\documentclass{article}

\usepackage{microtype}
\usepackage{graphicx}
\usepackage{subcaption}
\usepackage{booktabs} 
\usepackage{multicol}
\usepackage{multirow}
\usepackage{mathtools}

\usepackage{hyperref}

\usepackage[preprint]{icml2026}

\usepackage{amsmath}
\usepackage{amssymb}
\usepackage{mathtools}
\usepackage{amsthm}
\usepackage{booktabs}
\usepackage{siunitx}
\usepackage[table]{xcolor}

\usepackage[capitalize,noabbrev]{cleveref}

\theoremstyle{plain}
\newtheorem{theorem}{Theorem}[section]

\theoremstyle{definition}

\theoremstyle{remark}

\usepackage[textsize=tiny]{todonotes}

\icmltitlerunning{StaTS: Spectral Trajectory Schedule Learning for Adaptive Time Series Forecasting with Frequency Guided Denoiser}

\begin{document}

\twocolumn[
  \icmltitle{StaTS: Spectral Trajectory Schedule Learning for Adaptive Time Series Forecasting with Frequency Guided Denoiser}



  \icmlsetsymbol{equal}{*}

  \begin{icmlauthorlist}
    \icmlauthor{Jintao Zhang}{ustc}
    \icmlauthor{Zirui Liu}{ustc}
    \icmlauthor{Mingyue Cheng}{ustc}
    \icmlauthor{Xianquan Wang}{ustc}
    \icmlauthor{Zhiding Liu}{ustc}
    \icmlauthor{Qi Liu}{ustc}
  \end{icmlauthorlist}

  \icmlaffiliation{ustc}{State Key Laboratory of Cognitive Intelligence, University of Science and Technology of China, Hefei, China}

  \icmlcorrespondingauthor{Qi Liu}{qiliuql@ustc.edu.cn}
  \icmlcorrespondingauthor{Mingyue Cheng}{mycheng@ustc.edu.cn}
  \icmlcorrespondingauthor{Jintao Zhang}{zjttt@mail.ustc.edu.cn}

  \icmlkeywords{Machine Learning, ICML}

  \vskip 0.3in
]



\printAffiliationsAndNotice{}  

\begin{abstract}

Diffusion models have been used for probabilistic time series forecasting and show strong potential. However, fixed noise schedules often produce intermediate states that are hard to invert and a terminal state that deviates from the near noise assumption. Meanwhile, prior methods rely on time domain conditioning and seldom model schedule induced spectral degradation, which limits structure recovery across noise levels. We propose StaTS, a diffusion model for probabilistic time series forecasting that learns the noise schedule and the denoiser through alternating updates. StaTS includes Spectral Trajectory Scheduler (STS) that learns a data adaptive noise schedule with spectral regularization to improve structural preservation and stepwise invertibility, and Frequency Guided Denoiser (FGD) that estimates schedule induced spectral distortion and uses it to modulate denoising strength for heterogeneous restoration across diffusion steps and variables. A two stage training procedure stabilizes the coupling between schedule learning and denoiser optimization. Experiments on multiple real world benchmarks show consistent gains, while maintaining strong performance with fewer sampling steps. Our code is available at \url{https://github.com/zjt-gpu/StaTS/}.

\end{abstract}

\vspace{-0.25in}

\section{Introduction}

\begin{figure}[t]
  \centering
  \includegraphics[width=\linewidth]{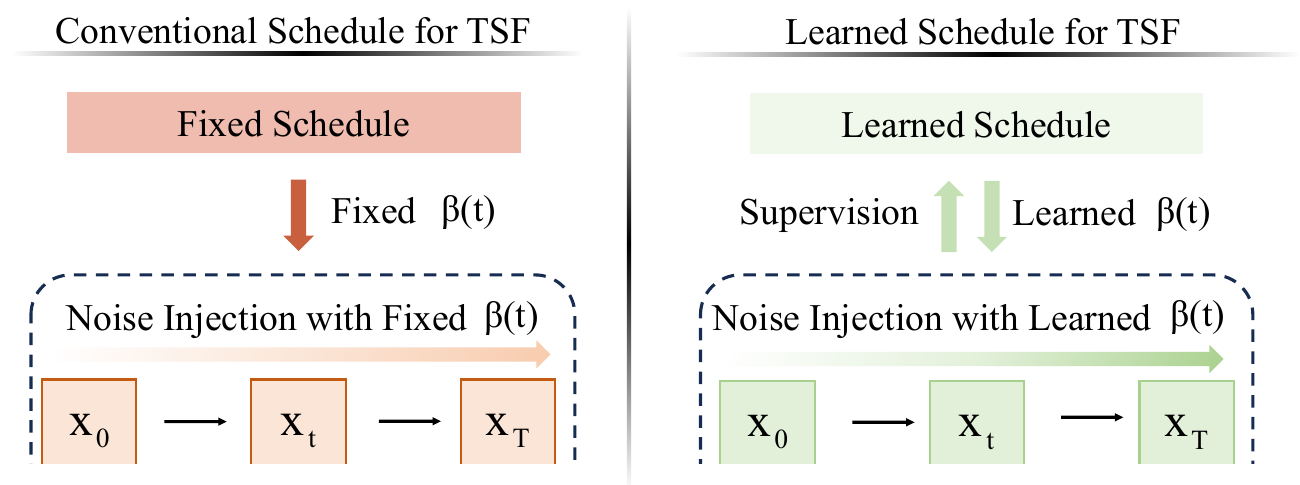}
  \caption{Conventional diffusion relies on fixed schedule, whereas our approach learns a data adaptive schedule, yielding more distinguishable intermediate states and better controlled terminal noise.
  }
  \label{fig:intro1}
  \vspace{-0.25in}
\end{figure}

Time series forecasting is vital across many domains~\cite{intro101}, including finance~\cite{intro1} and healthcare~\cite{cheng2024hmf}, where forecast accuracy directly affects decision making and risk management~\cite{intro3}, the ability to quantify future uncertainty by producing a distribution over plausible outcomes is essential.

In recent years, generative frameworks such as denoising diffusion probabilistic models (DDPMs) have been introduced to time series forecasting, enabling the modeling of complex conditional distributions via a Markovian process governed by a diffusion schedule~\cite{ddpm,intro5}. The schedule shapes the distribution of intermediate noisy states and therefore directly affects their denoisability~\cite{R24}. Without explicit constraints on the schedule, (i) intermediate noisy states may undergo spectral collapse or show weak spectral separation across steps, which hinders the denoiser from learning a stable stepwise inverse mapping~\cite{intro105}. (ii) The terminal noise level may depart from the diffusion assumption of being close to pure noise, which amplifies the mismatch between training and inference state distributions and yields unstable inversion~\cite{R23}. Therefore, an appropriate noise schedule for each dataset is crucial for the diffusion based forecasting~\cite{ant}.

Meanwhile, improving the denoiser ability to recover structure from noisy sequences is equally important~\cite{intro102,intro3}. Most diffusion based forecasting methods mainly rely on conditional inputs in the time domain to enhance representations~\cite{progen,timedart}, yet they do not explicitly characterize how trends, periodic components, and stochastic noise are degraded across different frequency bands. Modeling in the frequency domain separates these components into an interpretable and controllable spectral space~\cite{fredomain}, offering a more global condition for denoising.

To address these issues, we propose StaTS, which jointly optimizes the forward corruption trajectory and the reverse denoising process for diffusion forecasting.
As shown on the right of Fig.~\ref{fig:intro1}, the Spectral Trajectory Scheduler (STS) learns a data adaptive noise schedule rather than relying on fixed linear or cosine templates.
STS regularizes schedule learning with lightweight spectral constraints that encourage a flat terminal noise state, a smooth evolution of spectral flatness over steps, and a well behaved schedule that avoids degenerate near zero noise levels, while being directly optimized for forecasting performance. On the reverse process, the Frequency Guided Denoiser (FGD) explicitly estimates schedule induced spectral damage and uses it as a guidance signal to modulate denoising strength, enabling adaptive allocation of restoration effort across steps and noise levels. Experiments demonstrate that StaTS achieves consistent gains in probabilistic forecasting performance.


\begin{itemize}
    \item We propose StaTS, a diffusion based framework for probabilistic time series forecasting that couples noise scheduling with denoising to better align corruption and restoration.

    \item We develop STS that learns a data adaptive diffusion schedule regularized by spectral objectives, and FGD that estimates schedule induced spectral distortion and uses it to modulate denoising strength.
    
    \item Extensive experiments show that StaTS outperforms strong baselines, and maintains high accuracy and uncertainty quality under reduced sampling steps.

\end{itemize}

\begin{figure*}[t]
  \centering
  \includegraphics[width=\textwidth]{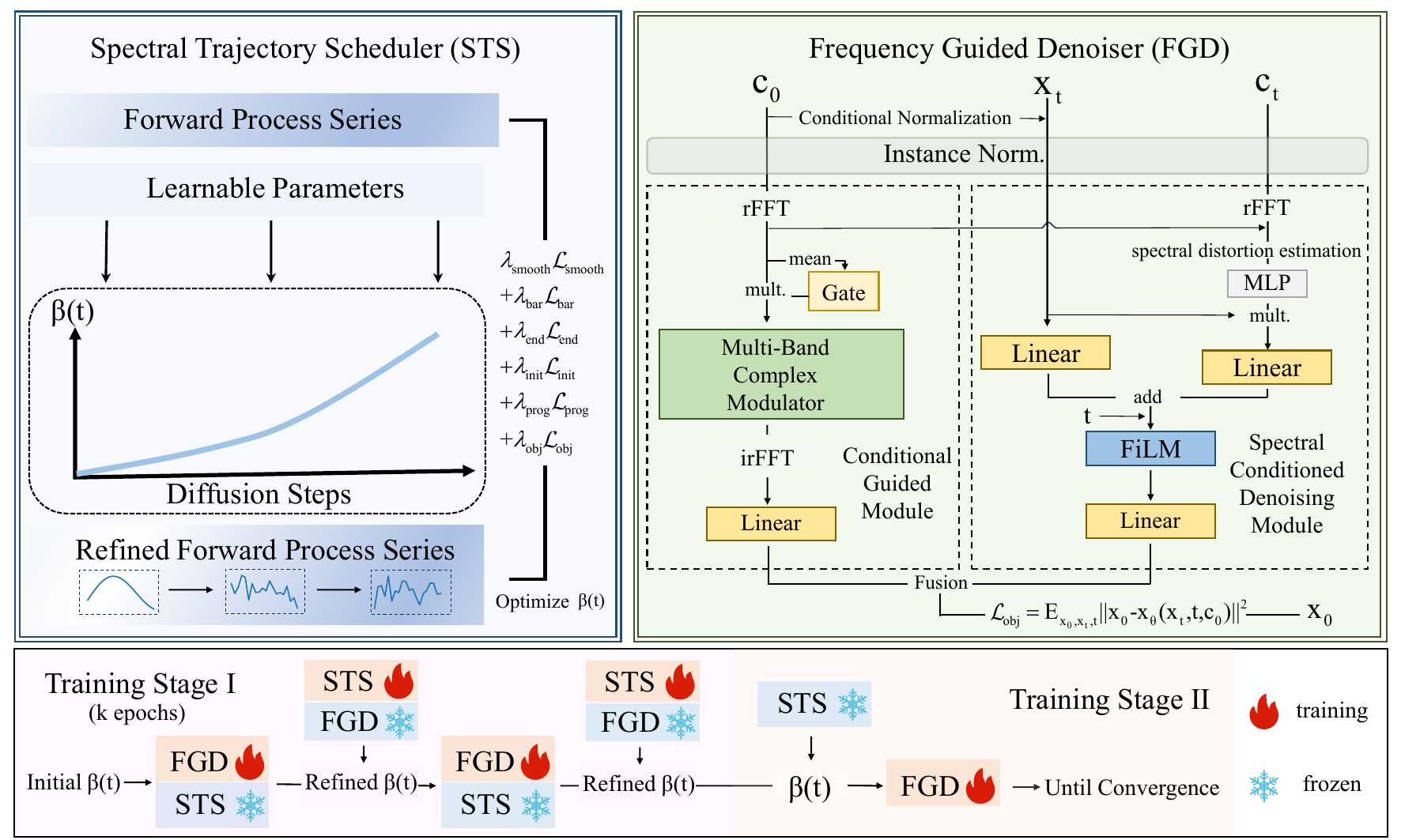}
     \caption{\textbf{Overview of StaTS.} The \textbf{Spectral Trajectory Scheduler (STS)} learns an adaptive variance schedule $\beta(t)$ to govern forward corruption and regularize the \emph{spectral flatness trajectory} over diffusion time. The \textbf{FGD} includes a conditional guidance module for encoding the history and a frequency denoising module for frequency-aware denoising under noise corruption. StaTS is trained in two stages: we alternately update STS and FGD for $k$ epochs, then freeze STS and train FGD to convergence with the learned schedule.}
  \label{fig:mainpic}
  \vspace{-0.15in}
\end{figure*}


\section{Problem Statement}

Let $\mathbf{X}\in\mathbb{R}^{(L+H)\times d}$ be a multivariate time series with history window $\mathbf{c}_0=\mathbf{X}^{1:L}$ and target window $\mathbf{x}_0=\mathbf{X}^{L+1:L+H}$. We model the conditional predictive distribution $p(\mathbf{x}_0\mid \mathbf{c}_0)$ with a $T$-step conditional diffusion process $\{\mathbf{x}_t\}_{t=0}^{T}$, where $\alpha_t=1-\beta_t$ and $\bar{\alpha}_t=\prod_{s=1}^{t}\alpha_s$. 

The forward process is defined as
\begin{equation}
q(\mathbf{x}_t\mid \mathbf{x}_{t-1})
=\mathcal{N}\!\left(\sqrt{\alpha_t}\mathbf{x}_{t-1},(1-\alpha_t)\mathbf{I}\right),
\end{equation}

For conditional generation, we learn a denoiser $\mathbf{x}_\theta(\mathbf{x}_t,t,\mathbf{c}_0)$ that predicts $\mathbf{x}_0$ given the noisy input and the history context. 

The reverse transition is parameterized by
\begin{equation}
p_\theta(\mathbf{x}_{t-1}\mid \mathbf{x}_t,\mathbf{c}_0)
=\mathcal{N}\!\left(\boldsymbol\mu_\theta(\mathbf{x}_t,t,\mathbf{c}_0),\sigma_t^2\mathbf{I}\right),
\end{equation}

We adopt the $\mathbf{x}_0$ prediction objective~\cite{JIT}, as predicting the clean signal $\mathbf{x}_0$ near data manifold is easier to learn and yields more stable optimization than predicting unstructured noise. The optimization objective is:

\begin{equation}
\mathcal{L}_{\mathrm{obj}}=\mathbb{E}_{t,\mathbf{x}_0,\mathbf{x}_t\sim q(\mathbf{x}_t\mid \mathbf{x}_0)}
\Big[\big\|\mathbf{x}_0-\mathbf{x}_\theta(\mathbf{x}_t,t,\mathbf{c}_0)\big\|_2^2\Big].
\end{equation}

This objective defines the learning objective.


\section{Methodology}

As shown in Fig.~\ref{fig:mainpic}, we present the overall architecture of StaTS and outline the two stage training strategy used to optimize STS and FGD.

\subsection{The proposed StaTS}
We propose a diffusion forecasting framework StaTS that comprises the Spectral Trajectory Scheduler (STS) module and the Frequency Guided Denoiser (FGD) module. 

\subsubsection{Spectral Trajectory Scheduler}

Unlike standard diffusion models with a predefined schedule, we parameterize $\beta_t$ as a learnable function of time and \textbf{optimize the noising trajectory to produce corrupted states that are maximally amenable to accurate denoising}. The forward noising chain is governed by a learnable trajectory schedule.

Proposition 1: Validity of the adaptive schedule~\cite{ddpm}.
Assuming $\beta_t\in(0,1)$ for all $t$, it follows that $\alpha_t=1-\beta_t\in(0,1)$ and hence
$\bar{\alpha}_t=\prod_{s=1}^{t}\alpha_s\in(0,1)$.
Moreover, $\{\bar{\alpha}_t\}_{t=1}^{T}$ is strictly decreasing since
\begin{equation}
\bar{\alpha}_{t+1}=\bar{\alpha}_t\,\alpha_{t+1}<\bar{\alpha}_t.
\end{equation}

We parameterize the variance schedule as a function of the diffusion step.
For each $t\in\{1,\ldots,T\}$, we encode the step using the step embedding $\mathbf{s}_t\in\mathbb{R}^{d_s}$,
and predict the noise level with a lightweight MLP $f$:
\begin{equation}
\beta(t)=\mathrm{clamp}\!\left(\sigma\!\left(f(\mathbf{s}_t)\right),\ \varepsilon,\ 1-\varepsilon\right),
\end{equation}
which ensures $\beta(t)\in(0,1)$.

\paragraph{Optimizing Target.}
We optimize the scheduler parameters $\phi$ with an auxiliary objective that shapes the corruption trajectory in the frequency domain. Concretely, it (i) promotes a spectrally flat terminal state, (ii) enforces a smooth spectral flatness progression across intermediate steps, and (iii) regularizes the noise schedule to be smooth while avoiding degenerate solutions.

\emph{Boundary Objective.}
To prevent the learned schedule from collapsing to zero, which would confine the denoiser to being trained on nearly identical noise levels, we introduce a lower bound barrier that penalizes excessively small $\beta(t)$:
\begin{equation}
\mathcal{L}_{\mathrm{bar}}
=
-\frac{1}{T-1}\sum_{t=2}^{T}\log(\beta(t)).
\end{equation}

\emph{Endpoint Objectives.}
To obtain a spectrally flat terminal state $\mathbf{x}_T$, we minimize the KL divergence between the normalized spectral mass of $\mathbf{x}_T$ and the uniform distribution.
Given $\mathbf{x}_0$, we obtain the terminal state $\mathbf{x}_T$ and compute its one-sided power spectrum through $\mathrm{rFFT}(\cdot)$.
Let $p_T$ be the normalized spectral mass over $K$ frequency bins,
\begin{equation}
p_T(f)=\frac{S_T(f)}{\sum_{j=1}^{K}S_T(j)},\qquad f=1,\ldots,K,
\end{equation}
where $S_T(f)\ge 0$ denotes the variable-averaged power at frequency bin $f$.
Specifically, we set $u(f)=1/K$ and define the endpoint objective as:
\begin{equation}
\mathcal{L}_{\mathrm{end}}=D_{\mathrm{KL}}(p_T\|u)
=\sum_{f=1}^{K} p_T(f)\log\!\frac{p_T(f)}{u(f)}.
\end{equation}

To prevent an excessively large initial noise level, we penalize the first step variance:
\begin{equation}
\mathcal{L}_{\mathrm{init}}=\beta(1)^2.
\end{equation}

\emph{Spectral Flatness Objective.}
To enlarge the spectral separability along the forward trajectory, we regularize the evolution of spectral flatness. Let $\mathrm{SF}(\cdot)$ be spectral flatness statistic, defined in formula~\ref{flatness}. For diffusion step $t$, we define $\gamma_t=t/T$ and the target flatness is: 
\begin{equation}
\mathrm{Interp}(t)
=
(1-\gamma_t)\,\mathrm{SF}(\mathbf{x}_0)+\gamma_t\,\mathrm{SF}(\mathbf{x}_T).
\end{equation}
We penalize deviations of the forward sample $\mathbf{x}_t$:
\begin{equation}
\mathcal{L}_{\mathrm{prog}}
=
\mathbb{E}_{t,\mathbf{x}_0,\mathbf{x}_t}
\Big[
\big\|
\mathrm{SF}(\mathbf{x}_t)-\mathrm{Interp}(t)
\big\|_2^2
\Big].
\end{equation}

\emph{Schedule Objective.}
We further impose a smoothness penalty and an initialization term,
\begin{equation}
\mathcal{L}_{\mathrm{smooth}}=\sum_{t=2}^{T}(\beta(t)-\beta(t-1))^2.
\end{equation}

In addition, to enhance the modeling capacity of denoiser, we optimize the schedule using the objective $\mathcal{L}_{\mathrm{obj}}$ induced by the frozen FGD, thereby shaping the learned corruption trajectory in a way that improves forecasting performance.

\noindent
The overall STS objective is
\begin{equation}
\small
\begin{aligned}
\mathcal{L}_{\mathrm{STS}}
= {}& \lambda_{\mathrm{bar}}\mathcal{L}_{\mathrm{bar}}
+\lambda_{\mathrm{end}}\mathcal{L}_{\mathrm{end}}
+\lambda_{\mathrm{init}}\mathcal{L}_{\mathrm{init}}\\
& \hspace{-5mm}  +\lambda_{\mathrm{prog}}\mathcal{L}_{\mathrm{prog}}
+\lambda_{\mathrm{smooth}}\mathcal{L}_{\mathrm{smooth}} 
+\lambda_{\mathrm{obj}}\mathcal{L}_{\mathrm{obj}},
\end{aligned}
\end{equation}
where $\{\lambda_{\cdot}\}$ are nonnegative scalar hyperparameters that balance the loss terms.

\subsubsection{Frequency Guided Denoiser}

Given the history context $\mathbf{c}_0$ and the noisy target state $\mathbf{x}_t$ at diffusion step $t$, FGD predicts the clean target $\hat{\mathbf{x}}_0$~\cite{JIT}. Guided by the corruption trajectory learned by STS, FGD leverages the spectral dynamics of the noisy history window as a prior. Specifically, the conditional guidance module distills cues of frequency domain, while the spectral conditioned denoising module recovers structural information through frequency aware denoising. This complementary mechanism synergistically enhances both forecasting accuracy and the fidelity of uncertainty quantification.

\paragraph{Instance Normalization.}
We apply instance normalization using statistics computed from the history window $\mathbf{c}_0$: for each channel $i$, we compute the mean and standard deviation of $\mathbf{c}_0^{(i)}$ over the temporal dimension and use them as conditional statistics to standardize both the history and target windows. The diffusion process is performed in this standardized space, and we invert the transform after denoising to restore the original scale~\cite{revin}.

\paragraph{Conditional Guided Module.}
In parallel, we develop a lightweight frequency module designed to generate a deterministic anchor forecast based on the history window.
\begin{equation}
\hat{\mathbf{x}}_0^{\mathrm{freq}}=\mathbf{x}_{\psi}(\mathbf{c}_0).
\end{equation}
The module operates explicitly in the Fourier domain.
We compute the one sided spectrum for each channel.
\begin{equation}
\mathbf{C}_f=\mathrm{rFFT}(\mathbf{c}_0)\in\mathbb{C}^{F\times d}.
\end{equation}
To distill salient frequencies, an magnitude-based gate $\mathbf{g}_f \in (0,1)^F$ is employed, where $F = \lfloor L/2 \rfloor + 1$ denotes the number of frequency bins. For each bin $f \in \{1, \ldots, F\}$, the log-magnitude statistic is defined as:
\begin{equation}
e(f)=\log\!\Big(1+\mathrm{AvgCh}\big(|\mathbf{C}_f(f,\cdot)|\big)\Big),
\end{equation}
where $\mathrm{AvgCh}(\cdot)$ averages over variables.
Stacking $\{e(f)\}_{f=1}^{F}$ yields $E\in\mathbb{R}^{F}$.
We then apply a learnable per-frequency affine map followed by a sigmoid function:
\begin{equation}
\mathbf{g}_f=\sigma\!\big(a\odot E +b\big)\in(0,1)^F,
\end{equation}
where $a,b\in\mathbb{R}^{F}$ are learnable parameters and $\odot$ denotes the element wise product.

To enable frequency selective modulation of both amplitude and phase, we employ a multi band complex reweighting mechanism.
Let $\{\mathcal{I}_b\}_{b=1}^{B}$ be a partition of the frequency index set $\{1,\ldots,F\}$ into $B$ disjoint bands, and denote $F_b = |\mathcal{I}_b|$.
For each band $b$, we extract the band spectrum
$\mathbf{C}_f^{(b)}=\mathbf{C}_f[\mathcal{I}_b,:]\in\mathbb{C}^{F_b\times d}$
and the corresponding frequency gate
$\mathbf{g}_f^{(b)}=\mathbf{g}_f[\mathcal{I}_b]\in(0,1)^{F_b}$.
We further introduce a learnable complex gain vector $h_b\in\mathbb{C}^{F_b}$ to parameterize bandwise amplitude scaling and phase shifting.
The reweighted spectrum for band $b$ is computed as
\begin{equation}
\tilde{\mathbf{C}}_f^{(b)}
=
\big(h_b \odot \mathbf{g}_f^{(b)}\big)\odot \mathbf{C}_f^{(b)},
\end{equation}
where $\big(h_b \odot \mathbf{g}_f^{(b)}\big)$ is broadcast along the channel dimension.
Concatenating these spectra yields the result:
\begin{equation}
\tilde{\mathbf{C}}_f
=
\mathrm{Concat}_{b=1}^{B}\!\left(\tilde{\mathbf{C}}_f^{(b)}\right)
\in\mathbb{C}^{F\times d}.
\end{equation}
We then apply the inverse to recover a filtered history signal
\begin{equation}
\bar{\mathbf{c}}=\mathrm{iRFFT}(\tilde{\mathbf{C}}_f)\in\mathbb{R}^{L\times d}.
\end{equation}
Finally, a linear projection outputs the prediction $\hat{\mathbf{x}}_0 \in \mathbb{R}^{H \times d}$.
This module combines learnable gating and multi-band complex scaling to form an expressive spectral filter.

\paragraph{Spectral Conditioned Denoising Module.}
Given the noisy target state $\mathbf{x}_t\in\mathbb{R}^{H\times d}$ at diffusion step $t$, the denoiser predicts the clean target as
\begin{equation}
\hat{\mathbf{x}}_0^{\mathrm{diff}}
=
\mathbf{x}_\theta(\mathbf{x}_t,t,\mathbf{c}_0),
\end{equation}
which constitutes the main stochastic generation pathway.
To couple denoising with the schedule induced spectral evolution, we explicitly estimate the frequency-domain distortion on the conditioning history and use it as the guidance.

Let $\mathbf{c}_t$ denote the corrupted history at diffusion step $t$, generated under the same schedule as the noisy target state $\mathbf{x}_t$.
We compute the spectra along the temporal axis using $\mathrm{rFFT}(\cdot)$. We then estimate the spectral distortion as
\begin{equation}
\mathbf{r}_t
=
\mathrm{clip}\!\left(
\frac{\big|\mathrm{rFFT}(\mathbf{c}_t)\big|-\big|\mathrm{rFFT}(\mathbf{c}_0)\big|}
{\big|\mathrm{rFFT}(\mathbf{c}_0)\big|+\epsilon},
\ r_{\min},\, r_{\max}
\right),
\end{equation}
where $\mathbf{r}_t\in\mathbb{R}^{F\times d}$ and $|\cdot|$ denotes the complex magnitude, $\epsilon$ is a small constant, and $\mathrm{clip}(\cdot)$ bounds the ratio.

A lightweight projection MLP maps $\mathbf{r}_t$ into the gate
\begin{equation}
\mathbf{g}_t
=
\sigma\!\big(\mathrm{MLP}(\mathbf{r}_t)\big),
\end{equation}
where $\sigma(\cdot)$ denotes the sigmoid function.
The gate is then broadcast and used to modulate the noisy target window
\begin{equation}
\tilde{\mathbf{x}}_t=\mathbf{x}_t\odot \mathbf{g}_t .
\end{equation}

We implement $\mathbf{x}_\theta$ as a compact conditioned backbone that processes each variable trajectory as an $H$-dimensional vector.
In this form, two Linear layers are applied to the raw and guided inputs, and their outputs are summed to yield the initial hidden state $\mathbf{h}_0$.

The diffusion step information is incorporated through a FiLM modulation mechanism~\cite{film}
\begin{equation}
\mathbf{h}
=
\phi\!\big(\mathrm{FiLM}(\mathbf{h}_0,t)\big),
\end{equation}
where $\phi(\cdot)$ denotes the SiLU activation function. In practice, the FiLM mechanism is applied twice. These FiLM conditioned activations are composed sequentially to yield the adaptive representation of different steps.

Finally, a refinement stage followed by a Linear head yields the denoised prediction $\hat{\mathbf{x}}_0^{\mathrm{diff}} \in \mathbb{R}^{H \times d}$.

\paragraph{Fusion.}
Final prediction is obtained through adaptive combination of conditioned prediction and denoised sample
\begin{equation}
\hat{\mathbf{x}}_0
=
\omega\,\hat{\mathbf{x}}_0^{\mathrm{freq}}
+(1-\omega)\,\hat{\mathbf{x}}_0^{\mathrm{diff}},
\end{equation}
where $\omega\in(0,1)$ is a learnable scalar that controls the relative contribution of the two components. We optimize the FGD network by minimizing $\mathcal{L}_{\mathrm{obj}}$.

\subsection{Training Strategy of StaTS}

A two stage training strategy is adopted to stabilize the coupled optimization of STS and FGD. 
In Stage One, coordinate wise alternating updates are performed for $k$ epochs. With the noise schedule $\beta(t)$ fixed, FGD is trained under the induced forward process. With FGD fixed, STS is optimized to refine $\beta(t)$. This alternation progressively aligns the forward corruption trajectory with the denoiser capacity.
In Stage Two, STS is frozen and the learned $\beta(t)$ is used to define a stationary forward process. FGD is then trained to convergence under this fixed corruption distribution, which eliminates forward process drift and stabilizes optimization. 
For distributional consistency, STS is learned in the same instance normalized space as FGD.

\begin{theorem} PGD convergence for schedule optimization. \label{thm:proj_conv} \end{theorem}
\vspace{-0.2in}
Let $\mathcal{B}=[\beta_{\min},\beta_{\max}]^{T}$ and assume $R$ is differentiable and $L$-smooth on a neighborhood of $\mathcal{B}$.
Consider the update
\[
\boldsymbol\beta^{k+1}=\mathrm{Proj}_{\mathcal{B}}\!\left(\boldsymbol\beta^{k}-\eta\nabla R(\boldsymbol\beta^{k})\right),
\qquad 0<\eta\le \tfrac{1}{L},
\]
where $\mathrm{Proj}_{\mathcal{B}}(\cdot)=\mathrm{clamp}(\cdot,\beta_{\min},\beta_{\max})$.
Define the projected gradient mapping
\[
G_{\eta}(\boldsymbol\beta)=\tfrac{1}{\eta}\!\left(\boldsymbol\beta-\mathrm{Proj}_{\mathcal{B}}(\boldsymbol\beta-\eta\nabla R(\boldsymbol\beta))\right).
\]
Then the objective decreases monotonically and
\[
R(\boldsymbol\beta^{k+1})\le R(\boldsymbol\beta^{k})-\tfrac{\eta}{2}\|G_{\eta}(\boldsymbol\beta^{k})\|^{2},
\qquad 
\|G_{\eta}(\boldsymbol\beta^{k})\|\to 0.
\]
In particular, any limit point $\boldsymbol\beta^\star$ satisfies the projected first-order stationarity condition on $\mathcal{B}$.

\begin{theorem} Stable forward drift under schedule updates.  \label{thm:stability} \end{theorem}
\vspace{-0.2in}
Fix $t\in\{1,\ldots,T\}$ and consider schedules
$\boldsymbol\beta,\boldsymbol\beta'\in[\beta_{\min},\beta_{\max}]^{T}\subset(0,1)^{T}$.
Let the forward process be
\begin{equation}\label{eq:forward_cond_beta}
q_{\boldsymbol\beta}(\mathbf{x}_t\mid\mathbf{x}_0)
=
\mathcal{N}\!\left(\sqrt{\bar{\alpha}_t}\,\mathbf{x}_0,\ (1-\bar{\alpha}_t)\mathbf{I}\right),
\end{equation}
\begin{equation}\label{eq:forward_cond_betap}
q_{\boldsymbol\beta'}(\mathbf{x}_t\mid\mathbf{x}_0)
=
\mathcal{N}\!\left(\sqrt{\bar{\alpha}'_t}\,\mathbf{x}_0,\ (1-\bar{\alpha}'_t)\mathbf{I}\right).
\end{equation}
Assume $\bar{\alpha}_t,\bar{\alpha}'_t\in[a,1-a]$ for $a\in(0,\tfrac12)$.
Then for any $\mathbf{x}_0$,
\begin{equation}\label{eq:kl_drift_bound}
D_{\mathrm{KL}}\!\left(q_{\boldsymbol\beta'}(\mathbf{x}_t\!\mid\!\mathbf{x}_0)\ \|\ q_{\boldsymbol\beta}(\mathbf{x}_t\!\mid\!\mathbf{x}_0)\right)
\le
C(a,t,\mathbf{x}_0)\,\|\boldsymbol\beta'-\boldsymbol\beta\|_{\infty}^{2},
\end{equation}
where $C(a,t,\mathbf{x}_0)$ is a constant independent of $\boldsymbol\beta,\boldsymbol\beta'$.

Theorem~\ref{thm:proj_conv} establishes the convergence of the projected schedule optimization under box constraints, ensuring that the updates monotonically decrease the objective toward a projected first-order stationary point. Theorem~\ref{thm:stability} demonstrates that bounded schedule perturbations result in controlled variations in the forward process. Detailed derivations and formal proofs are provided in Appendix~\ref{Theorem Evidence}.

\section{Experiments}

\subsection{Experimental Settings}

\paragraph{Datasets.} We evaluate our method on eight widely used real world multivariate time-series benchmarks with heterogeneous characteristics, including Electricity (ECL), ILI, ETT{h1, ETTh2, ETTm1, ETTm2}, Traffic and SolarEnergy. Table~\ref{tab:dataset_stats} reports the basic statistics. For data splitting, we follow standard protocols in time series forecasting. Specifically, ETT datasets are partitioned by 12/4/4 months for train/validation/test, while the remaining datasets adopt a 7:1:2 split~\cite{TMDM, patchtst}. Additional details can be found in Appendix~\ref{dataset}.

\begin{table}[thb]
\vspace{-0.1in}
\centering
\caption{Properties across different datasets.}
\label{tab:dataset_stats}
\setlength{\tabcolsep}{3pt}
\renewcommand{\arraystretch}{1.1}
\scriptsize
\resizebox{\columnwidth}{!}{%
\begin{tabular}{lcccccccc}
\hline
\textbf{Property} & \textbf{ETTh1} & \textbf{ETTh2} & \textbf{ETTm1} & \textbf{ETTm2} & \textbf{ECL} & \textbf{ILI} & \textbf{SolarEnergy} & \textbf{Traffic} \\
\hline
\textbf{Dimension} & 7 & 7 & 7 & 7 & 321 & 7 & 137 & 862 \\
\textbf{Steps}     & 17420 & 17420 & 69680 & 69680 & 26304 & 966 & 52560 & 17544 \\
\textbf{H}         & 192 & 192 & 192 & 192 & 192 & 36 & 192 & 192 \\
\hline
\end{tabular}%
}
\vspace{-0.2in}
\end{table}

\begin{table*}[t]
  \centering
  \caption{Probabilistic forecasting performance on long term multivariate datasets.}
  \vspace{-0.1in}
  \label{tab:prob_results_main}
  \setlength{\tabcolsep}{4pt}
  \renewcommand{\arraystretch}{1.15}
  \small
  \begin{tabular}{llcccccccc}
    \toprule
    Method & Metric & ETTh1 & ETTh2 & ETTm1 & ETTm2 & Electricity & ILI & SolarEnergy & Traffic \\
    \midrule
    \multirow{3}{*}{CSDI}
      & CRPS & 0.5240 & 0.8494 & 0.5144 & 0.8724 & 0.5715 & 1.3959 & 0.4374 & 0.6126 \\
      & MAE  & 0.7294 & 1.1557 & 0.7156 & 1.1352 & 0.8013 & 1.6311 & 0.6875 & 0.7744 \\
      & MSE  & 0.9779 & 2.2883 & 0.9429 & 2.2511 & 0.9662 & 5.5810 & 0.7371 & 1.3859 \\
    \addlinespace[2pt]
    \multirow{3}{*}{D3VAE}
      & CRPS & 0.5601 & 1.1163 & 0.4771 & 0.7281 & 0.4086 & 1.0506 & 0.4118 & 0.4394 \\
      & MAE  & 0.7440 & 1.5193 & 0.5690 & 1.0188 & 0.4539 & 1.3594 & 0.4257 & 0.4304 \\
      & MSE  & 0.9635 & 3.4012 & 0.6385 & 1.6347 & 0.4083 & 4.1941 & 0.3901 & 0.8116 \\
    \addlinespace[2pt]
    \multirow{3}{*}{TimeDiff}
      & CRPS & 0.4683 & 0.4816 & 0.4646 & 0.3129 & 0.7522 & 0.9915 & 0.7105 & 0.7717 \\
      & MAE  & \underline{0.4824} & 0.4952 & 0.4781 & \underline{0.3266} & 0.7660 & \underline{1.0486} & 0.7242 & 0.7854 \\
      & MSE  & \underline{0.5170} & 0.4723 & \underline{0.5319} & 0.2647 & 0.8701 & \underline{3.3182} & 0.8104 & 1.3440 \\
    \addlinespace[2pt]
    \multirow{3}{*}{DiffusionTS}
      & CRPS & 0.6024 & 1.1732 & 0.5558 & 1.0135 & 0.4556 & 1.5928 & 0.5074 & 0.4968 \\
      & MAE  & 0.7738 & 1.4185 & 0.7247 & 1.2113 & 0.5303 & 1.7771 & 0.7505 & 0.6126 \\
      & MSE  & 1.0889 & 3.2864 & 1.0052 & 2.3092 & 0.5009 & 6.0448 & 0.7635 & 1.0797 \\
    \addlinespace[2pt]
    \multirow{3}{*}{NsDiff}
      & CRPS & \underline{0.4021} & \underline{0.3537} & \underline{0.3643} & \underline{0.2701} & \underline{0.2848} & \underline{0.9541} & \underline{0.2474} & \underline{0.3710} \\
      & MAE  & 0.5382 & \underline{0.4779} & \underline{0.4713} & 0.3425 & \underline{0.3048} & 1.1592 & \underline{0.3053} & \underline{0.3950} \\
      & MSE  & 0.6186 & \underline{0.4644} & 0.5727 & \underline{0.2597} & \underline{0.2127} & 3.6098 & \textbf{0.2469} & \underline{0.6657} \\
    \midrule
    \multirow{4}{*}{StaTS}
      & CRPS & \textbf{0.3491} & \textbf{0.3148} & \textbf{0.3008} & \textbf{0.2389} & \textbf{0.2474} & \textbf{0.8293} & \textbf{0.2210} & \textbf{0.3091}  \\
      & \cellcolor{lightgray!30}\scriptsize Impr.(\%) & \cellcolor{lightgray!30}\scriptsize 13.18 & \cellcolor{lightgray!30}\scriptsize 11.00 & \cellcolor{lightgray!30}\scriptsize 17.43 & \cellcolor{lightgray!30}\scriptsize 11.55 & \cellcolor{lightgray!30}\scriptsize 13.13 & \cellcolor{lightgray!30}\scriptsize 13.08 & \cellcolor{lightgray!30}\scriptsize 10.67 & \cellcolor{lightgray!30}\scriptsize 16.68 \\
      & MAE  & \textbf{0.4584} & \textbf{0.4170} & \textbf{0.3946} & \textbf{0.3109} & \textbf{0.3003} & \textbf{0.9485} & \textbf{0.2721} & \textbf{0.3664} \\
      & MSE  & \textbf{0.4921} & \textbf{0.3983} & \textbf{0.3841} & \textbf{0.2529} & \textbf{0.2123} & \textbf{2.3996} & \underline{0.2663} & \textbf{0.5368} \\
    \bottomrule
  \end{tabular}
  \vspace{-0.2in}
\end{table*}

\paragraph{Baselines.} We compare against five state-of-the-art probabilistic time series forecasting baselines, including CSDI~\cite{csdi}, D3VAE~\cite{d3vae}, TimeDiff~\cite{timediff}, DiffusionTS~\cite{diffusionts}, and NsDiff~\cite{nsdiff}.

\paragraph{Experimental settings.}
All experiments follow the standard protocol for long term multivariate time series forecasting, with a fixed input length of $168$. Hyperparameters are selected based on the best validation performance, and the final results are reported on the test set. By default, we use a learning rate of $10^{-3}$, a batch size of $32$, and a diffusion horizon of $T=50$ steps~\cite{csdi, diffusionts}. For our method, we initialize the schedule with $\beta_1=10^{-5}$ and $\beta_T=0.1$ and learn an adaptive variance schedule during training. We train each model for $50$ epochs using three random seeds ${1,2,3}$. At inference time, we draw $100$ samples per test instance to estimate the predictive distribution. All baseline methods are evaluated using their default configurations~\cite{nsdiff}. More details of experimental settings are provided in Appendix~\ref{setting}.

\paragraph{Evaluation metrics.}
We report mean squared error (MSE) and mean absolute error (MAE) to evaluate forecasting accuracy, and use the continuous ranked probability score (CRPS)~\cite{crps} to assess predictive uncertainty. Lower values indicate better performance. The detailed formulation is provided in Appendix~\ref{metrics}.


\subsection{Main Results}
Table~\ref{tab:prob_results_main} summarizes the probabilistic forecasting performance across eight multivariate benchmarks. StaTS consistently outperforms all baselines, achieving the best CRPS and MAE on every dataset and obtaining the lowest MSE in most cases. Compared to the best performing baseline, StaTS yields substantial and consistent CRPS improvements, underscoring its stronger ability to characterize predictive distributions. These gains suggest that a carefully optimized, adaptive corruption trajectory enables the denoiser to better preserve and leverage essential structural information, thereby enhancing overall distributional modeling. Furthermore, the concurrent MAE improvements indicate that stronger representations also translate into higher point forecasting accuracy. While StaTS reduces MSE on most benchmarks, it slightly underperforms \textit{NsDiff} on SolarEnergy. This discrepancy likely stems from the inherent characteristics of SolarEnergy, where the prevalence of near-zero values and relatively smooth dynamics make the MSE metric exceptionally sensitive to minor amplitude deviations. The consistent CRPS gains indicate that StaTS improves both probabilistic characterization and forecasting accuracy.

\paragraph{Instance visualization.} Fig.~\ref{CRPS_com1} and Fig.~\ref{CRPS_com2} qualitatively compare NsDiff and StaTS on ETTh1 and Electricity, visualizing the predictive mean and uncertainty bands under the same history window. On ETTh1, NsDiff produces an excessively smoothed mean that fails to capture the periodic oscillations, and its uncertainty band expands almost uniformly with limited sensitivity to local structural changes. In contrast, StaTS tracks the oscillatory dynamics more faithfully in both phase and amplitude, and its uncertainty band contracts and widens in accordance with the evolving structure. On the more regular Electricity dataset, StaTS exhibits enhanced precision and adaptability. Specifically, it produces narrower intervals during stable periods while extending coverage to encompass local deviations. Further visual comparisons are provided in Appendix~\ref{visualization}.


\begin{figure}[t]
  \centering
  \includegraphics[width=0.5\textwidth]{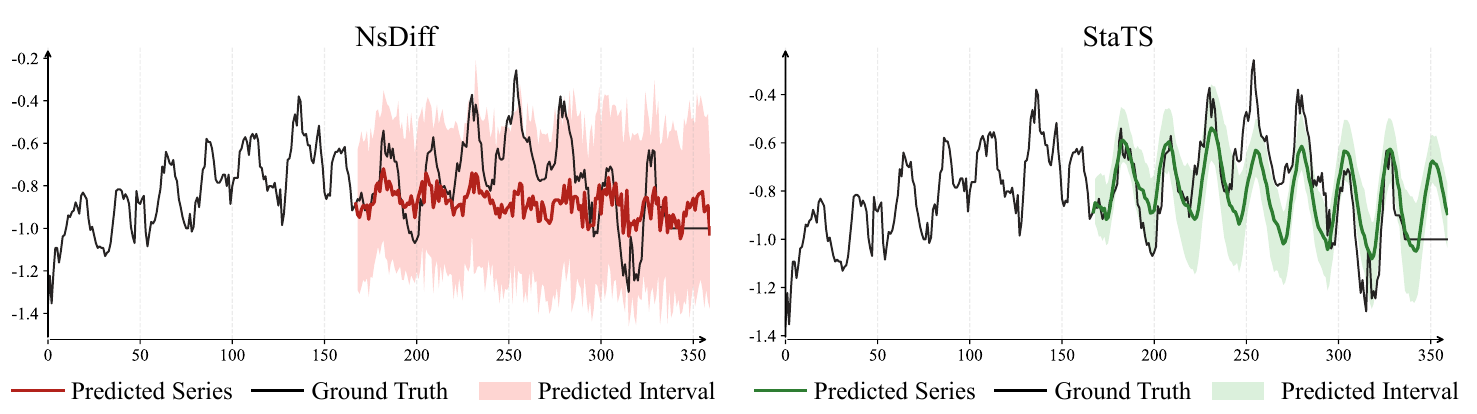}
     \caption{Visualize results on the ETTh1 dataset.}
     \vspace{-0.1in}
  \label{CRPS_com1}
\end{figure}

\begin{figure}[t]
  \centering
  \includegraphics[width=0.5\textwidth]{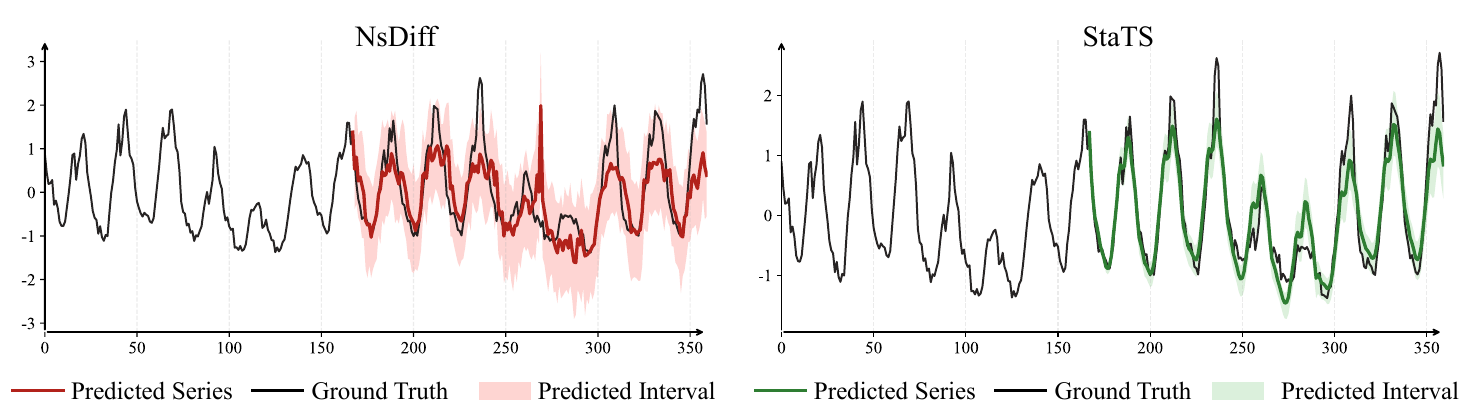}
     \caption{Visualize results on Electricity dataset.}
  \label{CRPS_com2}
  \vspace{-0.2in}
\end{figure}

\begin{table}[t]
  \centering
  \caption{CRPS under different diffusion steps $T$.}
    \vspace{-0.1in}
  \label{tab:steps_crps_compact}
  \footnotesize
  \setlength{\tabcolsep}{4pt}
  \renewcommand{\arraystretch}{1.12}

  \begin{tabular}{
    c
    S[table-format=1.4] S[table-format=1.4] c
    S[table-format=1.4] S[table-format=1.4] c
  }
    \toprule
    \multirow{2}{*}{$T$}
      & \multicolumn{3}{c}{ETTm1}
      & \multicolumn{3}{c}{Traffic} \\
    \cmidrule(lr){2-4}\cmidrule(lr){5-7}
      & {Linear} & {STS} & {\cellcolor{lightgray!30}Imp.(\%)}
      & {Linear} & {STS} & {\cellcolor{lightgray!30}Imp.(\%)} \\
    \midrule
    10  & 0.5031 & 0.3122 & \cellcolor{lightgray!30}37.9
        & 0.6738 & 0.3002 & \cellcolor{lightgray!30}55.4 \\
    20  & 0.4729 & 0.2951 & \cellcolor{lightgray!30}37.6
        & 0.4109 & 0.3131 & \cellcolor{lightgray!30}23.8 \\ 
    50  & 0.3102 & 0.3008 & \cellcolor{lightgray!30}3.0
        & 0.3154 & 0.3091 & \cellcolor{lightgray!30}2.0 \\
    100 & 0.3267 & 0.3047 & \cellcolor{lightgray!30}6.7
        & 0.3357 & 0.3333 & \cellcolor{lightgray!30}0.7 \\
    \bottomrule
  \end{tabular}
  \vspace{-0.2in}
\end{table}

\subsection{Effect of Diffusion Steps}

Table~\ref{tab:steps_crps_compact} investigates how the number of diffusion steps $T$ affects probabilistic forecasting performance under a linear schedule. Under a limited inference budget, STS yields the largest gains: with only a few reverse transitions, forecasting quality becomes highly sensitive to the denoisability of intermediate corrupted states. A fixed linear schedule can drive the forward process toward overly corrupted states and distort relevant spectral structure, making it difficult for the reverse process to recover informative details. In contrast, STS learns a corruption trajectory adapted to forecasting, thereby achieving markedly lower CRPS with small diffusion steps. As $T$ increases, the relative advantage of STS gradually diminishes and saturates. A longer reverse chain provides additional refinement opportunities, allowing the fixed schedule to progressively correct errors and better approximate the target distribution. Moreover, excessively large $T$ may introduce redundant transitions and accumulate estimation noise, further narrowing the gap across schedules. Further analyses can be found in Appendix~\ref{Cost sensitivity to diffusion steps}.

\begin{figure}[htbp]
  \centering
  \vspace{-0.1in}
  \includegraphics[width=\linewidth]{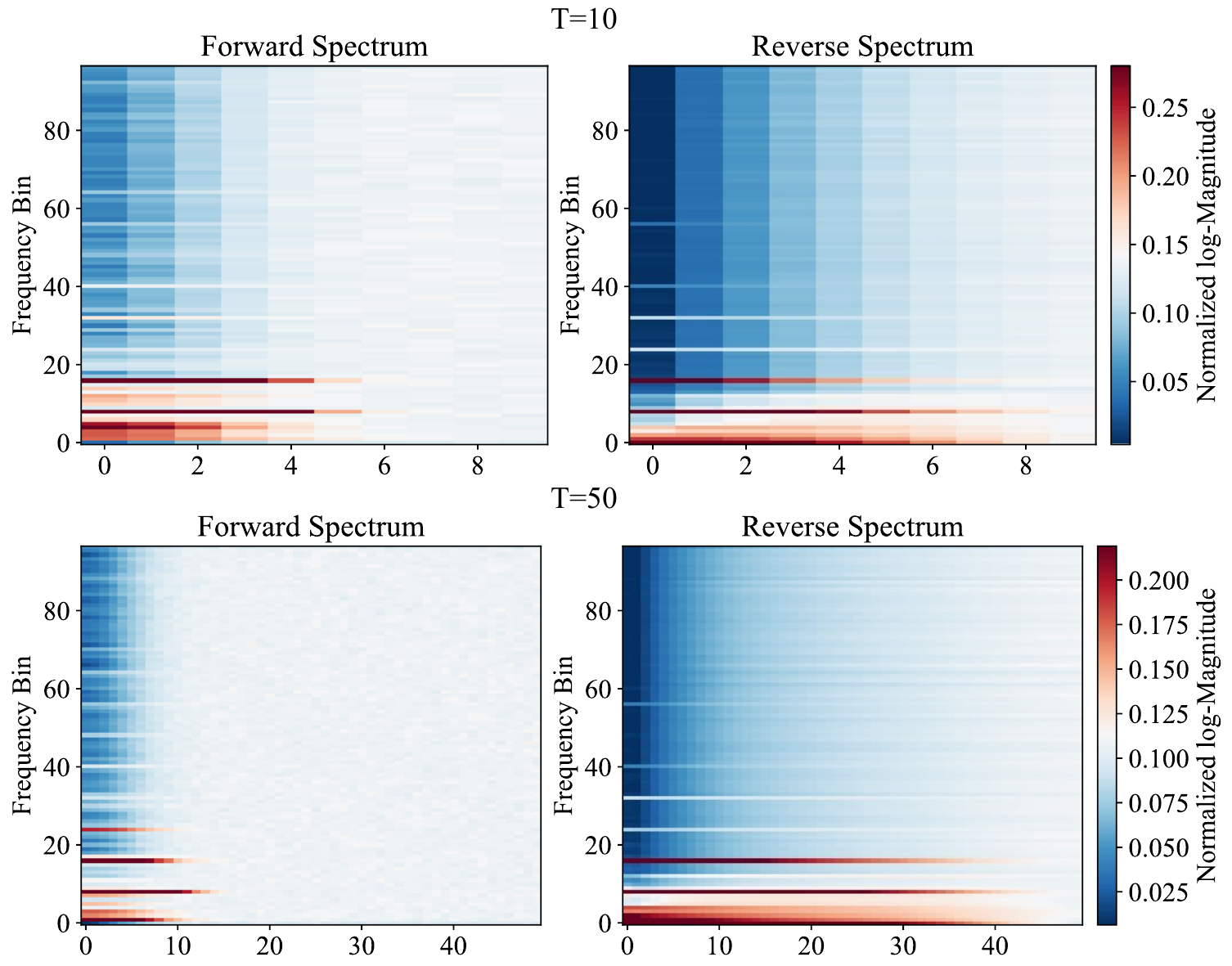}
  \caption{
  Normalized spectral evolution of the forward corruption and reverse denoising trajectories under different diffusion steps.
  }
  \label{fig:pingpu}
  \vspace{-0.2in}
\end{figure}

\paragraph{Controllability of spectral trajectories across diffusion steps.}
Fig.~\ref{fig:pingpu} visualizes the evolution of spectral energy along the forward corruption and reverse denoising processes under different numbers of diffusion steps $T$. The role of $T$ is not to decide whether recovery is possible, but to reshape how spectral structure is distributed and reconstructed along the chain. Under a small inference budget such as $T{=}10$, the forward process concentrates spectral compression into a few transitions, while the reverse chain has limited opportunities to refine and must recover structure from fewer intermediate states. As a result, performance becomes more sensitive to the denoisability of these states. When $T{=}50$, both spectral attenuation and subsequent recovery are distributed across more steps, producing a more gradual and continuous reconstruction. Crucially, in both regimes, the learned adaptive corruption trajectory improves the consistency between schedule driven spectral decay and reverse reconstruction dynamics, yielding more controllable spectral evolution and more stable conditional generation with better calibrated uncertainty.

\begin{figure}[t]
  \centering
  \includegraphics[width=0.5\textwidth]{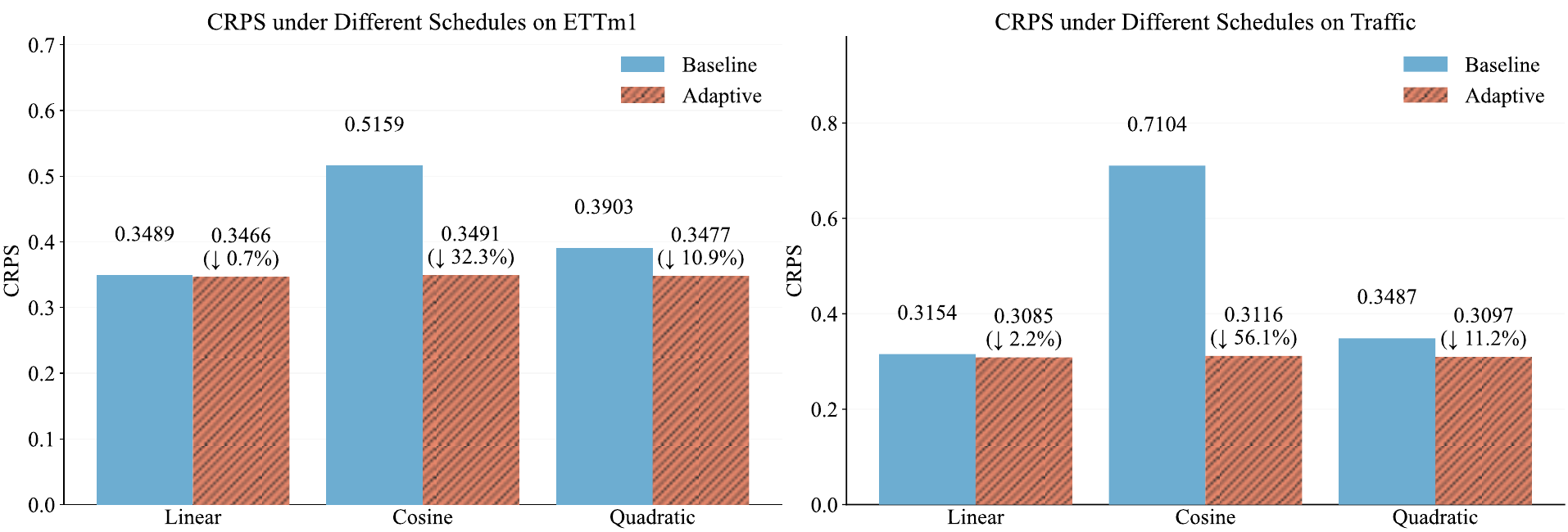}
     \caption{CRPS Performance under different schedules.}
  \label{CRPS_com}
  \vspace{-0.25in}
\end{figure}

\subsection{Effect of Diffusion Schedules}

Fig.~\ref{CRPS_com} compares forecasting performance under three schedule templates, Linear, Cosine, and Quadratic, with $T{=}50$. The fixed schedule exhibits clear sensitivity to the initialization. Different functional forms impose different corruption rates over diffusion time, which reshapes the distribution of intermediate noisy states and hence their denoisability. On strongly periodic series such as Traffic, an unsuitable initialization can create a mismatch between the forward corruption profile and the reverse denoiser correction capacity, leading to degraded probabilistic forecasts. After STS optimization, this dependence is substantially mitigated. The STS refined schedules consistently improve upon their corresponding schedules across all initializations, suggesting that STS learns an adaptive corruption trajectory that yields intermediate states that are easier to denoise. As a result, forecasting quality improves while the reliance on manual schedule tuning is reduced.

\begin{figure*}[t]
  \centering
  \includegraphics[width=\linewidth]{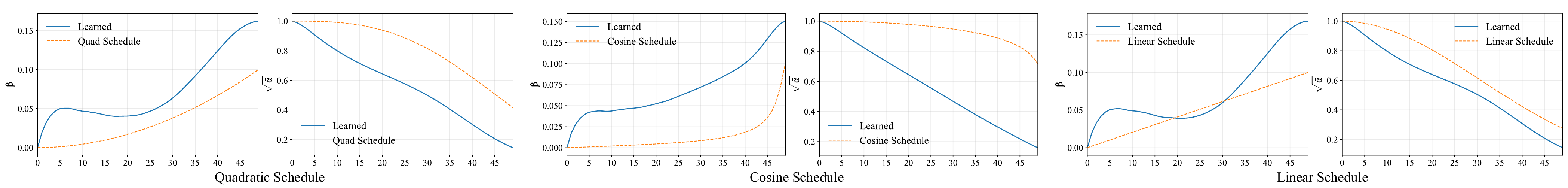}
  \vspace{-0.1in}
  \caption{
  Forward trajectories under different initialized schedules.
  }
  \label{fig:schedule_traj}
  \vspace{-0.1in}
\end{figure*}

\paragraph{Learned corruption trajectories across different schedule initializations.} Fig.~\ref{fig:schedule_traj} indicates that STS converges to a consistent schedule pattern regardless of the initialization on ETTh1 dataset. The learned $\beta_t$ departs from standard monotonic templates and takes a clearly nonlinear form, increasing sharply early on, flattening through the middle, and steepening again near the end. The corresponding $\bar{\alpha}_t$ therefore traces a noise accumulation trajectory that better reflects the statistics. This behavior suggests that STS learns a dataset-specific adaptive corruption trajectory. By tempering the noise injection rate in the intermediate phase, STS avoids compressing critical information into a small number of transitions, producing intermediate states.

\begin{table}[t]
  \centering
  \caption{Ablation study on forecasting performance.}
  \label{tab:ablation}
  \setlength{\tabcolsep}{6pt}
  \renewcommand{\arraystretch}{1.12}
  \small
  \begin{tabular}{l l c c c}
    \toprule
    Method & Metric & ETTm1 & Electricity & Traffic \\
    \midrule
    \multirow{3}{*}{w/o EPO}
      & CRPS & 0.5045 & 0.6748 & 0.6790 \\
      & MAE  & 0.5509 & 0.7440 & 0.7752 \\
      & MSE  & 0.6974 & 0.8181 & 1.3107 \\
    \addlinespace[2pt]
    \multirow{3}{*}{w/o SDE}
      & CRPS & 0.3035 & 0.2488 & 0.3169 \\
      & MAE  & 0.3994 & 0.3126 & 0.3795 \\
      & MSE  & 0.3952 & 0.2226 & 0.5566 \\
    \addlinespace[2pt]
    \multirow{3}{*}{w/o STS}
      & CRPS & 0.3102 & 0.2517 & 0.3154 \\
      & MAE  & 0.3986 & 0.3026 & 0.3675 \\
      & MSE  & 0.3861 & 0.2128 & 0.5371 \\
    \multirow{3}{*}{w/o IN}
      & CRPS & 0.4455 & 0.2585 & 0.3404 \\
      & MAE  & 0.5399 & 0.3020 & 0.3879 \\
      & MSE  & 0.5427 & 0.2150 & 0.5828 \\
    \midrule
    \multirow{3}{*}{\textbf{StaTS}}
      & CRPS & \textbf{0.3008} & \textbf{0.2474} & \textbf{0.3091} \\
      & MAE  & \textbf{0.3946} & \textbf{0.3003} & \textbf{0.3664} \\
      & MSE  & \textbf{0.3841} & \textbf{0.2123} & \textbf{0.5368} \\
    \bottomrule
  \end{tabular}
  \vspace{-0.25in}
\end{table}

\subsection{Ablation Study.}
Table~\ref{tab:ablation} reports the ablation results of StaTS on three datasets to assess the contribution of each component.

\emph{w/o EPO.} We remove the endpoint objectives $\mathcal{L}_{\mathrm{init}}$ and $\mathcal{L}_{\mathrm{end}}$ when optimizing STS.

\emph{w/o SDE.} We remove Spectral Distortion Estimation from the spectral conditioned denoising module.

\emph{w/o STS.} We replace the Spectral Trajectory Scheduler with the fixed linear schedule instead.

\emph{w/o IN.} We remove both instance normalization and conditional normalization modules.

Removing the schedule endpoint objectives causes the largest degradation, especially on Electricity and Traffic datasets, indicating that a well posed initialization and a spectrally meaningful terminal constraint are critical for shaping a stable corruption trajectory. Otherwise, the learned schedule becomes less regular and produces intermediate states that are harder to denoise, which in turn makes uncertainty estimation more challenging. Removing SDE consistently worsens CRPS and increases error of the point forecasting, suggesting that the denoiser benefits from the schedule induced spectral damage to modulate corrections under heterogeneous corruption. Disabling STS also reduces performance across benchmarks, the fixed schedule yields a spectral evolution that is less aligned with dataset characteristics, leading to intermediate states that are less compatible with the reverse process and thus degrading both distributional quality and point accuracy. Moreover, removing IN leads to a clear performance drop, with the most pronounced degradation on ETTm1. IN leverages statistics from the history window to align the feature distributions of the history and target windows, thereby mitigating amplitude shifts. Without normalization, denoising becomes less stable and the predicted distribution is less well calibrated.

\begin{figure}[t]
  \centering
  \includegraphics[width=0.5\textwidth]{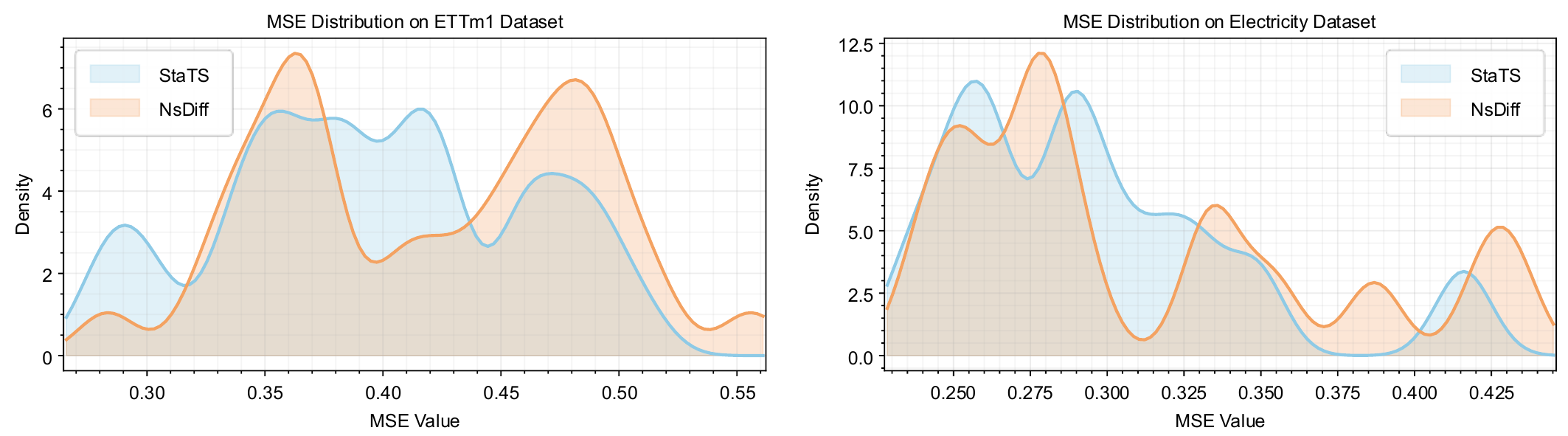}
     \caption{MSE distributions on the ETTm1 and Electricity datasets.}
  \label{fb}
  \vspace{-0.2in}
\end{figure}

\subsection{MSE distributions.} 
Fig.~\ref{fb} illustrates the MSE distributions across individual test instances on the ETTm1 and Electricity datasets~\cite{liu2025improving}. Compared to NsDiff, StaTS shows a clearer left shift and a higher degree of concentration on both datasets, indicating lower prediction errors for a larger portion of test cases. In contrast, NsDiff presents a more dispersed distribution with multiple modes and retains more probability mass in regions of higher error, suggesting reduced stability under heterogeneous samples. The lighter right tails and lower overall dispersion observed for StaTS indicate stronger reliability. This distributional pattern is consistent with the average results reported in Table~\ref{tab:prob_results_main} and helps explain the overall improvements in CRPS and MSE.

\section{Conclusion}
We propose StaTS, a diffusion forecasting framework that couples noise scheduling and denoising through a Spectral Trajectory Scheduler (STS) and a Frequency Guided Denoiser (FGD).
STS learns a data-adaptive schedule under spectral regularization, while FGD estimates schedule-induced spectral distortion to modulate denoising strength across diffusion steps and variables. Across long-horizon multivariate benchmarks, StaTS consistently outperforms prior methods, achieving concurrent improvements in probabilistic forecasting quality (CRPS) and point accuracy (MAE). Further analyses show that the spectrally regularized schedule yields intermediate states with stronger stepwise separability and improved invertibility, and that frequency guided denoising supports heterogeneous restoration while remaining effective with fewer inference steps. Future work will develop more efficient inference and step adaptive strategies while preserving forecasting quality.

\clearpage

\section*{Impact Statement}
This paper presents work whose goal is to advance the field of Machine Learning. There are many potential societal consequences of our work, none of which we feel must be specifically highlighted here.

\bibliography{example_paper}
\bibliographystyle{icml2026}

\clearpage

\appendix
\onecolumn

\section{Theorem Evidence}

\label{Theorem Evidence}

We provide a detailed proof of Theorem~3.1 below.

\begin{proof}
Let
\begin{equation}
\boldsymbol z^k = \boldsymbol\beta^{k}-\eta \nabla R(\boldsymbol\beta^{k}),\qquad
\boldsymbol\beta^{k+1}=\mathrm{Proj}_{\mathcal{B}}(\boldsymbol z^k),
\qquad
\Delta^k = \boldsymbol\beta^{k+1}-\boldsymbol\beta^{k}.
\end{equation}
We first use a standard optimality property of Euclidean projection onto a closed convex set.
Since $\mathcal{B}$ is closed and convex, $\boldsymbol\beta^{k+1}=\mathrm{Proj}_{\mathcal{B}}(\boldsymbol z^k)$ is the unique minimizer of
\begin{equation}
\min_{\boldsymbol u\in\mathcal{B}} \ \frac{1}{2}\|\boldsymbol u-\boldsymbol z^k\|^2.
\end{equation}
Therefore, the first-order optimality condition for this convex problem gives the variational inequality
\begin{equation}
\left\langle \boldsymbol\beta^{k+1}-\boldsymbol z^k,\ \boldsymbol u-\boldsymbol\beta^{k+1}\right\rangle \ge 0,
\qquad \forall\,\boldsymbol u\in\mathcal{B}.
\label{eq:proj_vi}
\end{equation}
Choosing $\boldsymbol u=\boldsymbol\beta^{k}\in\mathcal{B}$ in \eqref{eq:proj_vi} yields
\begin{align}
0 
&\le \left\langle \boldsymbol\beta^{k+1}-\boldsymbol z^k,\ \boldsymbol\beta^{k}-\boldsymbol\beta^{k+1}\right\rangle \nonumber\\
&= \left\langle \boldsymbol\beta^{k+1}-\boldsymbol\beta^{k}+\eta \nabla R(\boldsymbol\beta^{k}),\ \boldsymbol\beta^{k}-\boldsymbol\beta^{k+1}\right\rangle \nonumber\\
&= -\|\Delta^k\|^2 + \eta \left\langle \nabla R(\boldsymbol\beta^{k}),\ \boldsymbol\beta^{k}-\boldsymbol\beta^{k+1}\right\rangle.
\end{align}
Rearranging gives
\begin{equation}
\left\langle \nabla R(\boldsymbol\beta^{k}),\ \boldsymbol\beta^{k+1}-\boldsymbol\beta^{k}\right\rangle
\le
-\frac{1}{\eta}\|\Delta^k\|^2.
\label{eq:inner_bd}
\end{equation}

Next, use $L$-smoothness of $R$.
For any $\boldsymbol x,\boldsymbol y$ in the neighborhood where $R$ is $L$-smooth, the descent lemma holds:
\begin{equation}
R(\boldsymbol y)\le R(\boldsymbol x)+\langle \nabla R(\boldsymbol x),\boldsymbol y-\boldsymbol x\rangle+\frac{L}{2}\|\boldsymbol y-\boldsymbol x\|^2.
\label{eq:descent_lemma}
\end{equation}
Apply \eqref{eq:descent_lemma} with $\boldsymbol x=\boldsymbol\beta^{k}$ and $\boldsymbol y=\boldsymbol\beta^{k+1}$:
\begin{equation}
R(\boldsymbol\beta^{k+1})
\le
R(\boldsymbol\beta^{k})
+\left\langle \nabla R(\boldsymbol\beta^{k}),\Delta^k\right\rangle
+\frac{L}{2}\|\Delta^k\|^2.
\label{eq:rl_smooth_step}
\end{equation}
Combining \eqref{eq:rl_smooth_step} with \eqref{eq:inner_bd} yields
\begin{equation}
R(\boldsymbol\beta^{k+1})
\le
R(\boldsymbol\beta^{k})
-\left(\frac{1}{\eta}-\frac{L}{2}\right)\|\Delta^k\|^2.
\label{eq:basic_descent}
\end{equation}
Under the step-size condition $0<\eta\le \frac{1}{L}$,
\begin{equation}
\frac{1}{\eta}-\frac{L}{2}
=\frac{2-L\eta}{2\eta}
\ge
\frac{1}{2\eta}.
\end{equation}
Plugging this into \eqref{eq:basic_descent} gives
\begin{equation}
R(\boldsymbol\beta^{k+1})
\le
R(\boldsymbol\beta^{k})
-\frac{1}{2\eta}\|\Delta^k\|^2.
\label{eq:descent_delta}
\end{equation}

Now relate $\Delta^k$ to the projected gradient mapping.
By definition of $G_\eta$ and the PGD update,
\begin{equation}
G_\eta(\boldsymbol\beta^{k})
=
\frac{1}{\eta}\Big(\boldsymbol\beta^{k}-\boldsymbol\beta^{k+1}\Big)
=
-\frac{1}{\eta}\Delta^k,
\end{equation}
hence $\|\Delta^k\|^2=\eta^2\|G_\eta(\boldsymbol\beta^{k})\|^2$. Substituting this into \eqref{eq:descent_delta} yields
\begin{equation}
R(\boldsymbol\beta^{k+1})
\le
R(\boldsymbol\beta^{k})
-\frac{\eta}{2}\|G_\eta(\boldsymbol\beta^{k})\|^2,
\end{equation}
which proves the one-step descent inequality.

Summing the descent inequality from $k=0$ to $K$ gives the telescoping bound
\begin{equation}
\frac{\eta}{2}\sum_{k=0}^{K}\|G_\eta(\boldsymbol\beta^{k})\|^2
\le
R(\boldsymbol\beta^{0})-R(\boldsymbol\beta^{K+1})
\le
R(\boldsymbol\beta^{0})-\inf_{\boldsymbol\beta\in\mathcal{B}}R(\boldsymbol\beta).
\end{equation}
Letting $K\to\infty$ yields
\begin{equation}
\sum_{k=0}^{\infty}\|G_\eta(\boldsymbol\beta^{k})\|^2
\le
\frac{2}{\eta}\Big(R(\boldsymbol\beta^{0})-\inf_{\boldsymbol\beta\in\mathcal{B}}R(\boldsymbol\beta)\Big)
<
\infty.
\end{equation}
Since each term $\|G_\eta(\boldsymbol\beta^{k})\|^2\ge 0$ and the series converges, it follows that
\begin{equation}
\|G_\eta(\boldsymbol\beta^{k})\|\to 0.
\end{equation}

Finally, we show the equivalence between $G_\eta(\boldsymbol\beta^\star)=0$ and the projected first-order stationarity condition.
First, $G_\eta(\boldsymbol\beta^\star)=0$ is equivalent to
\begin{equation}
\boldsymbol\beta^\star=\mathrm{Proj}_{\mathcal{B}}(\boldsymbol\beta^\star-\eta\nabla R(\boldsymbol\beta^\star)).
\label{eq:fixed_point_proj}
\end{equation}
Applying the projection variational inequality \eqref{eq:proj_vi} at $\boldsymbol z^\star=\boldsymbol\beta^\star-\eta\nabla R(\boldsymbol\beta^\star)$ and $\boldsymbol\beta^\star=\mathrm{Proj}_{\mathcal{B}}(\boldsymbol z^\star)$ gives, for all $\boldsymbol u\in\mathcal{B}$,
\begin{equation}
\left\langle \boldsymbol\beta^\star-(\boldsymbol\beta^\star-\eta\nabla R(\boldsymbol\beta^\star)),\ \boldsymbol u-\boldsymbol\beta^\star\right\rangle
=
\eta\left\langle \nabla R(\boldsymbol\beta^\star),\ \boldsymbol u-\boldsymbol\beta^\star\right\rangle
\ge 0.
\end{equation}
Since $\eta>0$, this is equivalent to
\begin{equation}
\left\langle \nabla R(\boldsymbol\beta^\star),\ \boldsymbol u-\boldsymbol\beta^\star\right\rangle \ge 0,\qquad \forall\,\boldsymbol u\in\mathcal{B}.
\end{equation}
Conversely, if the above inequality holds for all $\boldsymbol u\in\mathcal{B}$, then $\boldsymbol\beta^\star$ satisfies the projection optimality condition for $\min_{\boldsymbol u\in\mathcal{B}}\frac{1}{2}\|\boldsymbol u-(\boldsymbol\beta^\star-\eta\nabla R(\boldsymbol\beta^\star))\|^2$, hence \eqref{eq:fixed_point_proj} holds, and therefore $G_\eta(\boldsymbol\beta^\star)=0$.

\end{proof}

Then we provide a detailed proof of Theorem~3.2 below.

\begin{proof}
Write
\begin{equation}
\mu=\sqrt{\bar{\alpha}_t}\,\mathbf{x}_0,\quad
\mu'=\sqrt{\bar{\alpha}'_t}\,\mathbf{x}_0,\quad
v=1-\bar{\alpha}_t,\quad
v'=1-\bar{\alpha}'_t.
\end{equation}
Then $q_{\boldsymbol\beta}(\mathbf{x}_t\mid \mathbf{x}_0)=\mathcal{N}(\mu,v\mathbf{I})$ and
$q_{\boldsymbol\beta'}(\mathbf{x}_t\mid \mathbf{x}_0)=\mathcal{N}(\mu',v'\mathbf{I})$ in $\mathbb{R}^{D}$.

For $p=\mathcal{N}(\mu',v'\mathbf{I})$ and $q=\mathcal{N}(\mu,v\mathbf{I})$,
\begin{equation}
D_{\mathrm{KL}}(p\|q)
=
\frac{1}{2}\left[
D\left(\frac{v'}{v}-1-\log\frac{v'}{v}\right)
+
\frac{1}{v}\|\mu'-\mu\|_2^2
\right].
\label{eq:kl_iso}
\end{equation}
Under $\bar{\alpha}_t,\bar{\alpha}'_t\in[a,1-a]$, we have $v,v'\in[a,1-a]$.

Let $r=\frac{v'}{v}$. Since $v,v'\in[a,1-a]$,
\begin{equation}
r\in\left[\frac{a}{1-a},\frac{1-a}{a}\right].
\end{equation}
Define $\phi(r)=r-1-\log r$. Then $\phi(1)=\phi'(1)=0$ and $\phi''(r)=1/r^{2}$.
By Taylor's theorem, for some $\xi$ between $1$ and $r$,
\begin{equation}
\phi(r)=\frac{\phi''(\xi)}{2}(r-1)^2=\frac{1}{2\xi^{2}}(r-1)^2
\le
\frac{1}{2}\left(\frac{1-a}{a}\right)^{2}(r-1)^2,
\label{eq:phi_bd}
\end{equation}
where we used $\xi\ge \frac{a}{1-a}$.
Moreover,
\begin{equation}
r-1=\frac{v'-v}{v}=\frac{(\bar{\alpha}_t-\bar{\alpha}'_t)}{v},
\qquad
|r-1|\le \frac{|\bar{\alpha}_t-\bar{\alpha}'_t|}{a},
\label{eq:r_minus_1}
\end{equation}
since $v\ge a$. Combining \eqref{eq:phi_bd} and \eqref{eq:r_minus_1} gives
\begin{equation}
\phi(r)
\le
\frac{(1-a)^2}{2a^{4}}\,|\bar{\alpha}_t-\bar{\alpha}'_t|^{2}.
\end{equation}
Plugging this into \eqref{eq:kl_iso} yields the bound for the first KL component:
\begin{equation}
\frac{1}{2}D\,\phi(r)
\le
\frac{D(1-a)^2}{4a^{4}}\,|\bar{\alpha}_t-\bar{\alpha}'_t|^{2}.
\label{eq:var_part}
\end{equation}

We have
\begin{equation}
\|\mu'-\mu\|_2^2
=
\left(\sqrt{\bar{\alpha}'_t}-\sqrt{\bar{\alpha}_t}\right)^2\|\mathbf{x}_0\|_2^2.
\end{equation}
Since $\bar{\alpha}_t,\bar{\alpha}'_t\in[a,1-a]$, function $g(u)=\sqrt{u}$ has derivative
$g'(u)=\frac{1}{2\sqrt{u}}\le \frac{1}{2\sqrt{a}}$ on $[a,1-a]$. By the mean value theorem,
\begin{equation}
\left|\sqrt{\bar{\alpha}'_t}-\sqrt{\bar{\alpha}_t}\right|
\le
\frac{1}{2\sqrt{a}}\,|\bar{\alpha}'_t-\bar{\alpha}_t|.
\end{equation}
Therefore,
\begin{equation}
\|\mu'-\mu\|_2^2
\le
\frac{1}{4a}\,|\bar{\alpha}'_t-\bar{\alpha}_t|^2\,\|\mathbf{x}_0\|_2^2.
\end{equation}
Using $v\ge a$ in \eqref{eq:kl_iso}, the second KL component is bounded by
\begin{equation}
\frac{1}{2}\cdot \frac{1}{v}\|\mu'-\mu\|_2^2
\le
\frac{1}{2}\cdot \frac{1}{a}\cdot \frac{1}{4a}\,|\bar{\alpha}'_t-\bar{\alpha}_t|^2\,\|\mathbf{x}_0\|_2^2
=
\frac{\|\mathbf{x}_0\|_2^2}{8a^{2}}\,|\bar{\alpha}'_t-\bar{\alpha}_t|^2.
\label{eq:mean_part}
\end{equation}

Recall $\bar{\alpha}_t=\prod_{s=1}^{t}(1-\beta_s)$ and $\bar{\alpha}'_t=\prod_{s=1}^{t}(1-\beta'_s)$.
Let
\begin{equation}
S=\log\bar{\alpha}_t=\sum_{s=1}^t \log(1-\beta_s),\qquad
S'=\log\bar{\alpha}'_t=\sum_{s=1}^t \log(1-\beta'_s).
\end{equation}
For each $s$, by the mean value theorem applied to $\log(1-u)$, there exists $\tilde{\beta}_s$ between $\beta_s$ and $\beta'_s$ such that
\begin{equation}
\left|\log(1-\beta'_s)-\log(1-\beta_s)\right|
=
\frac{|\beta'_s-\beta_s|}{1-\tilde{\beta}_s}
\le
\frac{|\beta'_s-\beta_s|}{1-\beta_{\max}}.
\end{equation}
Summing over $s=1,\ldots,t$ gives
\begin{equation}
|S'-S|
\le
\frac{t}{1-\beta_{\max}}\,\|\boldsymbol\beta'-\boldsymbol\beta\|_\infty.
\label{eq:log_alpha_bd}
\end{equation}
Finally, since $\bar{\alpha}_t,\bar{\alpha}'_t\in(0,1]$, we have $S,S'\le 0$ and hence $e^{\max(S,S')}\le 1$.
Using the mean value theorem for the exponential function,
\begin{equation}
|\bar{\alpha}'_t-\bar{\alpha}_t|
=
|e^{S'}-e^{S}|
=
e^{\xi}\,|S'-S|
\le
|S'-S|
\le
\frac{t}{1-\beta_{\max}}\,\|\boldsymbol\beta'-\boldsymbol\beta\|_\infty,
\label{eq:alpha_bd}
\end{equation}
for some $\xi$ between $S$ and $S'$.

Combining \eqref{eq:kl_iso}, \eqref{eq:var_part}, and \eqref{eq:mean_part} yields
\begin{equation}
D_{\mathrm{KL}}\!\left(q_{\boldsymbol\beta'}(\mathbf{x}_t\mid \mathbf{x}_0)\ \|\ q_{\boldsymbol\beta}(\mathbf{x}_t\mid \mathbf{x}_0)\right)
\le
\left[
\frac{D(1-a)^2}{4a^{4}}
+
\frac{\|\mathbf{x}_0\|_2^2}{8a^{2}}
\right]
|\bar{\alpha}'_t-\bar{\alpha}_t|^2.
\end{equation}
Applying \eqref{eq:alpha_bd} gives
\begin{equation}
D_{\mathrm{KL}}\!\left(q_{\boldsymbol\beta'}(\mathbf{x}_t\mid \mathbf{x}_0)\ \|\ q_{\boldsymbol\beta}(\mathbf{x}_t\mid \mathbf{x}_0)\right)
\le
\left(\frac{t}{1-\beta_{\max}}\right)^{2}
\left[
\frac{D(1-a)^{2}}{4a^{4}}
+
\frac{\|\mathbf{x}_0\|_{2}^{2}}{8a^{2}}
\right]
\|\boldsymbol\beta'-\boldsymbol\beta\|_{\infty}^{2},
\end{equation}
\end{proof}

\section{Experiment Details}

\subsection{Dataset}

\label{dataset}

\paragraph{Datasets.}
We conduct experiments on widely used real-world time series datasets, including (1) Electricity~\footnote{\url{https://archive.ics.uci.edu/ml/datasets/ElectricityLoadDiagrams20112014}}, which records hourly electricity consumption of 321 customers spanning 2012--2014. (2) ILI~\footnote{\url{https://gis.cdc.gov/grasp/fluview/fluportaldashboard.html}}, which reports the weekly ratio of influenza-like illness cases to total patient visits released by the U.S.\ CDC over 2002--2021. (3) ETT~\footnote{\url{https://github.com/zhouhaoyi/ETDataset}}~\cite{informer}, which includes data from electricity transformers, such as load and oil temperature, recorded every 15 minutes or hourly intervals between July 2016 and July 2018. (4) Traffic~\footnote{\url{https://www.nrel.gov/}}, which provides hourly freeway occupancy rates measured by 862 sensors in the San Francisco Bay Area during January 2015--December 2016. (5) SolarEnergy~\footnote{\url{http://www.nrel.gov/grid/solar-power-data.html}}, which contains solar power generation readings from 137 photovoltaic plants in Alabama collected in 2007.

\subsection{Metrics}

\label{metrics}

\paragraph{CRPS.}
We evaluate probabilistic forecasts using the continuous ranked probability score (CRPS) \citep{crps}, which measures the discrepancy between a predictive cumulative distribution function (CDF) $F$ and an observation $x$:
\begin{equation}
\mathrm{CRPS}(F,x)=\int_{\mathbb{R}}\big(F(z)-\mathbb{I}[z\ge x]\big)^2\,\mathrm{d}z,
\end{equation}
where $\mathbb{I}[\cdot]$ is the indicator function.
In practice, $F$ is not available in closed form, so we approximate it with $S$ samples drawn from the model and compute CRPS based on this Monte Carlo approximation. Unless otherwise specified, we use $S=100$ samples per forecast.

\paragraph{MAE and MSE.}
For point forecasting accuracy, we also report mean absolute error (MAE) and mean squared error (MSE).
Given the ground-truth sequence $\{x_t\}_{t=1}^{H}$ and the corresponding point prediction $\{\hat{x}_t\}_{t=1}^{H}$, the metrics are defined as
\begin{equation}
\mathrm{MAE}=\frac{1}{H}\sum_{t=1}^{H}\big|\hat{x}_t-x_t\big|,
\qquad
\mathrm{MSE}=\frac{1}{H}\sum_{t=1}^{H}\big(\hat{x}_t-x_t\big)^2.
\end{equation}
For multivariate forecasts, we compute the errors per variable and average them over variables and time steps.

\section{Other Results and Discussions}

\subsection{Spectral Flatness Evolution under Forward Process}
\label{subsec:spectral_flatness}

\paragraph{Computation of spectral flatness.}

\begin{figure}[htbp]
  \centering
  \includegraphics[width=0.8\linewidth]{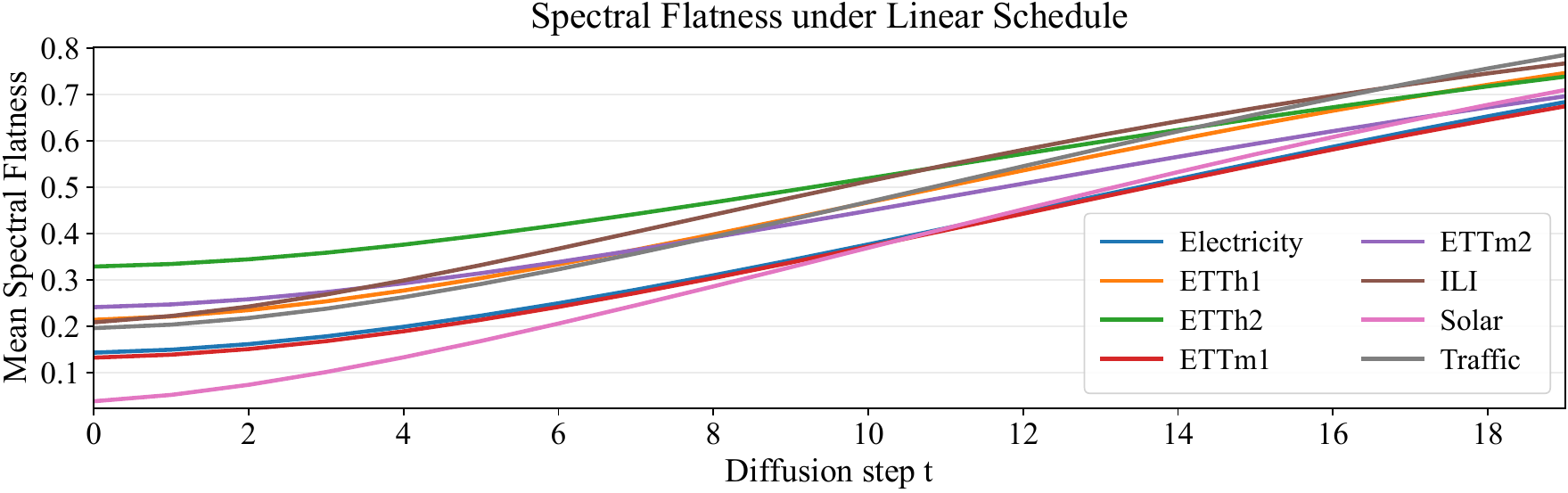}
  \caption{
  Evolution of mean spectral flatness along the forward diffusion trajectory under a linear noise schedule.
  }
  \label{fig:sfm_8datasets}
\end{figure}

Given $\mathbf{c}_t\in\mathbb{R}^{L\times d}$, we compute the one-sided spectrum
$\mathcal{F}(\mathbf{c}_t)=\mathrm{rFFT}(\mathbf{c}_t)\in\mathbb{C}^{F\times d}$
along the temporal axis, where $F=\lfloor L/2\rfloor+1$.
The variable-averaged power spectrum is
\begin{equation}
S_t(f)=\frac{1}{d}\sum_{i=1}^{d}\big|\mathcal{F}(\mathbf{c}_t)_{f,i}\big|^2,\qquad f=1,\ldots,F.
\end{equation}
and the normalized spectral mass is
\begin{equation}
p_t(f)=\frac{S_t(f)}{\sum_{j=1}^{F} S_t(j)}.
\end{equation}
We then define the spectral flatness statistic as the ratio between the geometric
and arithmetic means of the power spectrum:
\begin{equation}
\mathrm{SF}(\mathbf{c}_t)
=\frac{\exp\!\left(\frac{1}{F}\sum_{f=1}^{F}\log\big(S_t(f)+\epsilon\big)\right)}
{\frac{1}{F}\sum_{f=1}^{F}\big(S_t(f)+\epsilon\big)},\label{flatness}
\end{equation}
where $\epsilon>0$ is a small constant.
We measure spectral flatness by the KL divergence to the uniform distribution,
\begin{equation}
\mathcal{L}_{\mathrm{flat}}(t)
=
D_{\mathrm{KL}}\!\left(p_t(f)\,\big\|\,\mathcal{U}(f)\right),
\qquad
\mathcal{U}(f)=\frac{1}{F},
\end{equation}
which is applied at $t=T$ for the terminal flatness constraint.

To examine how forward diffusion alters the frequency structure of instance-normalized time series, we follow the evolution of spectral flatness along the noising trajectory. Spectral flatness increases when spectral energy becomes less concentrated, making it a concise indicator of schedule-induced degradation in the frequency domain.

\begin{figure}[htbp]
  \centering
  \includegraphics[width=0.8\linewidth]{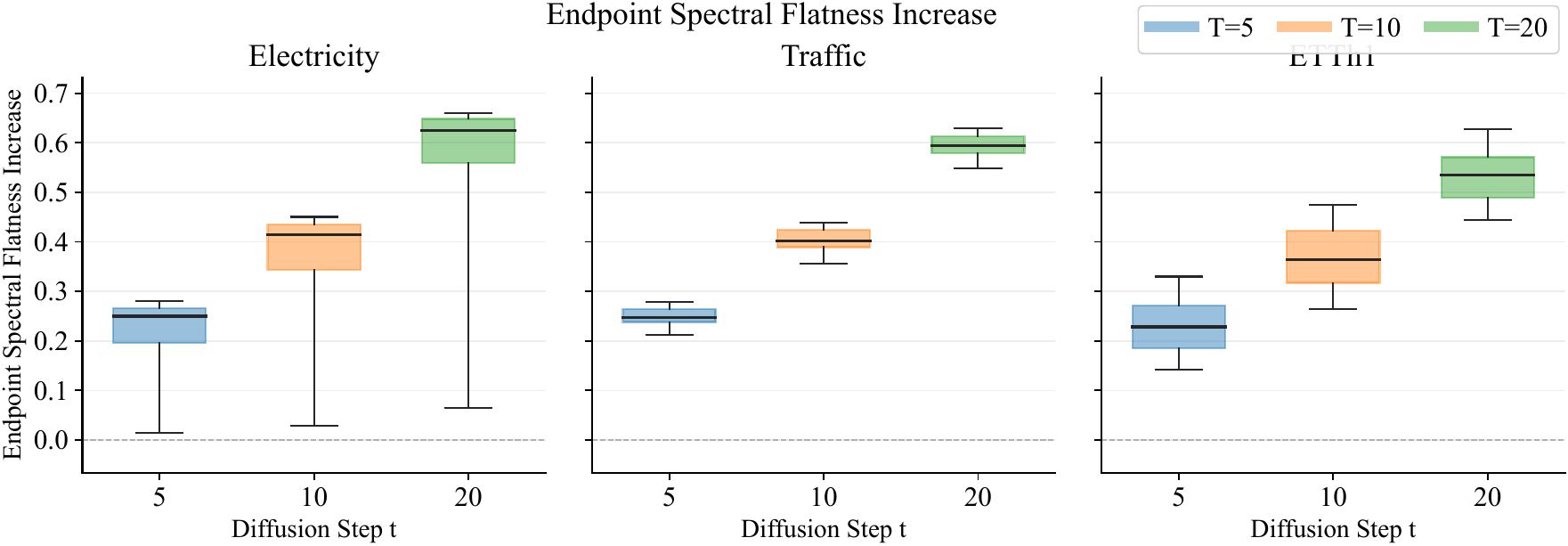}
  \caption{
  Increase in endpoint spectral flatness under different diffusion steps $T$.
  }
  \label{fig:endpoint_damage}
\end{figure}

Fig.~\ref{fig:sfm_8datasets} presents the spectral flatness trends over diffusion steps under a linear schedule on eight datasets. Across diverse domains and dynamics, spectral flatness exhibits a clear upward trend as the diffusion progresses, suggesting that noising generally drives sequences toward flatter spectra and weakens structured frequency patterns. Meanwhile, datasets differ in their starting levels and growth behaviors, revealing varying sensitivity to spectral corruption. In addition, the endpoint behavior is influenced by the diffusion length $T$. Fig.~\ref{fig:endpoint_damage} shows that longer diffusion trajectories tend to produce flatter endpoints across datasets, implying stronger disruption to the underlying frequency structure.

\subsection{Computation Efficiency}

\begin{table}[htbp]
  \centering
  \caption{Comparison of computational cost and forecasting accuracy on the Traffic dataset. Peak memory and each iteration runtime are reported for training and inference.}
  \label{tab:overall_cost_perf}
  \footnotesize
  \setlength{\tabcolsep}{5pt}
  \renewcommand{\arraystretch}{1.12}
  \resizebox{0.8\linewidth}{!}{%
  \begin{tabular}{
    l
    S[table-format=4.2]
    S[table-format=3.2]
    S[table-format=2.2]
    S[table-format=2.2]
    S[table-format=1.4]
    S[table-format=1.4]
    S[table-format=1.4]
  }
    \toprule
    {Model} &
    {Mem. Tr. (MB)} & {Mem. Inf. (MB)} &
    {Time Tr. (ms)} & {Time Inf. (ms)} &
    {CRPS} & {MAE} & {MSE} \\
    \midrule
    CSDI          & 3512.09 & 452.15 & 77.81 & 29.14 & 0.5240 & 0.7294 & 0.9779 \\
    TimeDiff      & 408.24  & 111.34 & 16.02 & \textbf{4.14}  & 0.4683 & 0.4824 & 0.5170 \\
    DiffusionTS   & 2069.95 & 235.86 & 61.12 & 19.39 & 0.6024 & 0.7738 & 1.0889 \\
    NsDiff        & 1526.68 & 493.07 & 34.10 & 15.71 & 0.4021 & 0.5382 & 0.6186 \\
    \textbf{StaTS} & \textbf{27.74} & \textbf{26.02} & \textbf{10.61} & 4.82
                   & \textbf{0.3491} & \textbf{0.4584} & \textbf{0.4921} \\
    \bottomrule
  \end{tabular}}
  \label{tab:overall_cost_perf}
\end{table}

\paragraph{Overall efficiency and accuracy comparison.}
Table~\ref{tab:overall_cost_perf} summarizes both computational cost and forecasting quality on Traffic. 
StaTS consistently delivers the strongest forecasting performance while maintaining substantially lower memory consumption than diffusion-based baselines during both training and inference. 
This result indicates that StaTS improves distributional forecasting quality while maintaining low runtime and memory overhead compared to diffusion-style methods.

\begin{table}[htbp]
  \centering
  \caption{Computational efficiency of StaTS across diffusion steps $T$ on ETTm1 and Traffic. Peak memory and each iteration time are reported for training and inference.}
  \label{tab:stats_steps_cost}
  \footnotesize
  \setlength{\tabcolsep}{3.5pt}
  \renewcommand{\arraystretch}{1.08}

  \resizebox{0.9\linewidth}{!}{%
  \begin{tabular}{c
                  S[table-format=2.4] S[table-format=2.4] S[table-format=1.4] S[table-format=1.4]
                  @{\hspace{6pt}}
                  S[table-format=3.4] S[table-format=2.4] S[table-format=1.4] S[table-format=1.4]}
    \toprule
    & \multicolumn{4}{c}{\textbf{ETTm1}} & \multicolumn{4}{c}{\textbf{Traffic}} \\
    \cmidrule(lr){2-5}\cmidrule(lr){6-9}
    {$T$} &
    {Mem. Tr. (MB)} & {Mem. Inf. (MB)} & {Time Tr. (ms)} & {Time Inf. (ms)} &
    {Mem. Tr. (MB)} & {Mem. Inf. (MB)} & {Time Tr. (ms)} & {Time Inf. (ms)} \\
    \midrule
    10  & 27.2852 & 25.4229 & 5.9771 & 2.0699 & 148.4170 & 97.5752 & 6.7862 & 1.6952 \\
    20  & 27.4082 & 25.5830 & 5.5442 & 2.1560 & 148.5342 & 97.7314 & 4.7156 & 1.7873 \\
    50  & 27.7832 & 26.0674 & 6.0914 & 2.1784 & 148.8887 & 98.2021 & 5.6521 & 1.8300 \\
    100 & 28.4043 & 26.8721 & 5.3611 & 2.1985 & 149.4805 & 98.9873 & 6.3710 & 1.8538 \\
    \bottomrule
  \end{tabular}%
  }
  \label{tab:stats_steps_cost}
\end{table}

\paragraph{Cost sensitivity to diffusion steps.}
\label{Cost sensitivity to diffusion steps}
Table~\ref{tab:stats_steps_cost} further examines how the computational cost of StaTS changes with the diffusion steps \(T\) on ETTm1 and Traffic. 
Across \(T\in\{10,20,50,100\}\), peak memory and time per iteration remain largely stable, indicating that StaTS is not strongly affected by the choice of step number in practice. 
In addition, the difference between the two datasets is mainly driven by their scale, with Traffic requiring more memory than ETTm1, while the trends across \(T\) are consistent on both datasets.

\begin{figure}[htbp]
  \centering
  \includegraphics[width=0.8\linewidth]{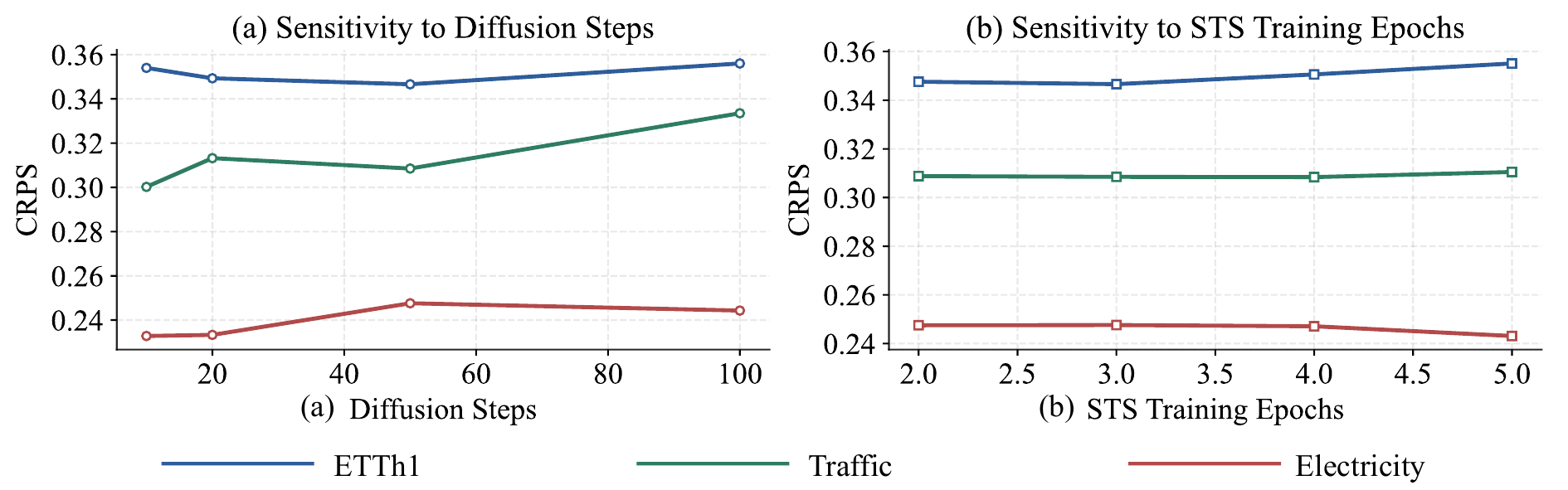}
  \caption{
  Parameter Sensitivity analysis of key hyperparameters. 
  }
  \label{fig:para_ana}
\end{figure}

\subsection{Parameters Sensitivity Analysis}

Fig.~\ref{fig:para_ana} evaluates the sensitivity of our framework to two key hyperparameters: the diffusion sampling steps $T$ and the number of training epochs used for the Spectral Trajectory Scheduler (STS).

\paragraph{Sensitivity to diffusion steps $T$.}
As shown in Fig.~\ref{fig:para_ana} (a), varying $T$ over $\{10, 20, 50, 100\}$ leads to only mild fluctuations in CRPS across all three datasets, without a consistent monotonic trend. This indicates that the proposed method is not overly sensitive to the sampling granularity. In particular, ETTh1 attains its best performance at a moderate $T$, while Traffic exhibits a slightly larger yet still controlled variation as $T$ increases. Electricity remains nearly unchanged throughout the tested range. Overall, these results suggest that stable probabilistic forecasting quality can be achieved without carefully tuning the diffusion step number.

\paragraph{Sensitivity to STS training epochs.}
Fig.~\ref{fig:para_ana} (b) further studies the impact of the training epochs allocated to STS. When increasing the STS training epochs from 2 to 5, the CRPS values remain largely stable on all datasets, with only marginal changes and consistent relative trends. This behavior suggests that STS converges reliably and does not require extensive training or delicate epoch tuning to learn an effective spectral scheduling trajectory. In practice, a small number of STS training epochs is sufficient to obtain robust performance, which reduces the hyperparameter tuning burden and improves the usability of our approach.

\subsection{Visualization}
\label{visualization}

Fig.~\ref{fig:visualization_etth1} and Fig.~\ref{fig:visualization_traffic} visualize the min-max prediction intervals together with the corresponding point forecasts on representative cases. Overall, several baselines either produce overly smoothed mean trajectories that fail to track the phase and amplitude of the target dynamics, or rely on excessively wide intervals to cover the ground truth, which weakens the practical interpretability of uncertainty.

On ETTh1 dataset, where the future evolution is relatively unstable and exhibits pronounced local fluctuations, StaTS maintains a more stable alignment with the underlying waveform, avoiding the drift or collapse patterns observed in competing methods. On Traffic dataset, StaTS achieves accurate tracking with noticeably tighter intervals, providing more informative uncertainty bounds and hence more actionable forecasts for real-world decision making.

\begin{figure}[htbp]
  \centering
  \includegraphics[width=0.9\linewidth]{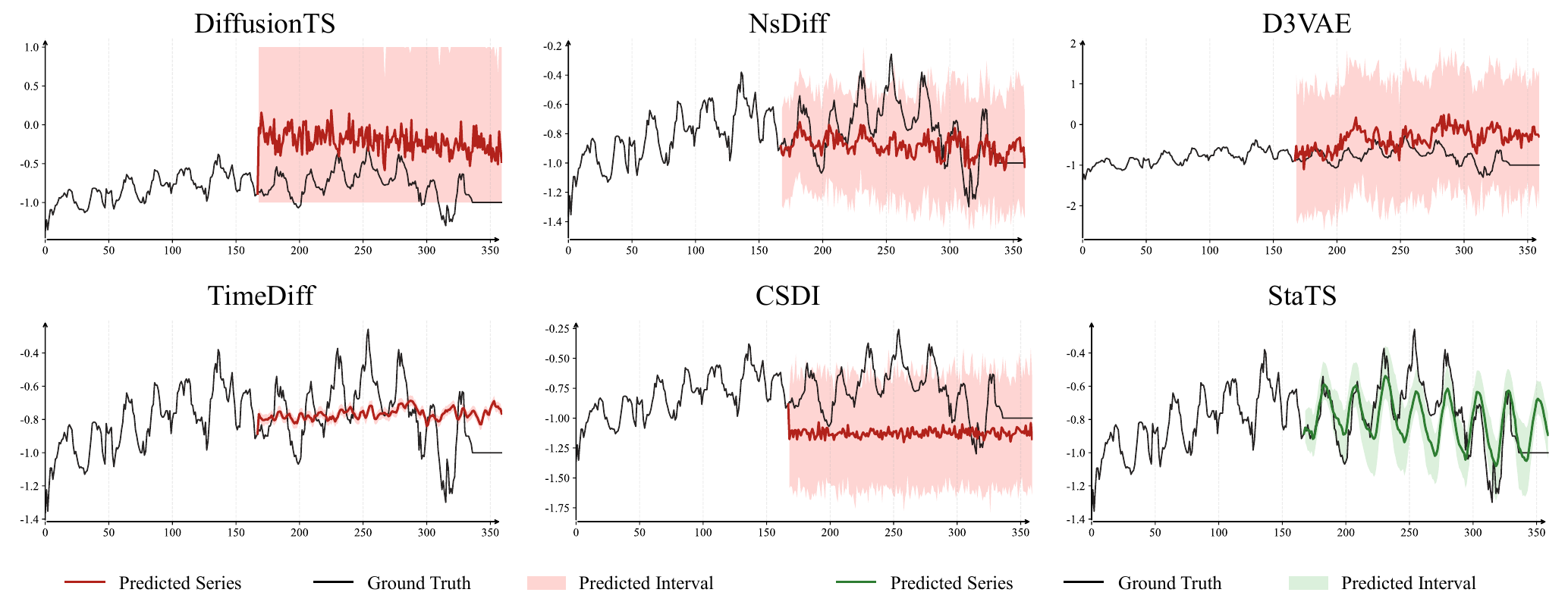}
  \caption{
  Visualization of min--max prediction intervals and point forecasts on the ETTh1 dataset.
  }
  \label{fig:visualization_etth1}
\end{figure}

\begin{figure}[htbp]
  \centering
  \includegraphics[width=0.9\linewidth]{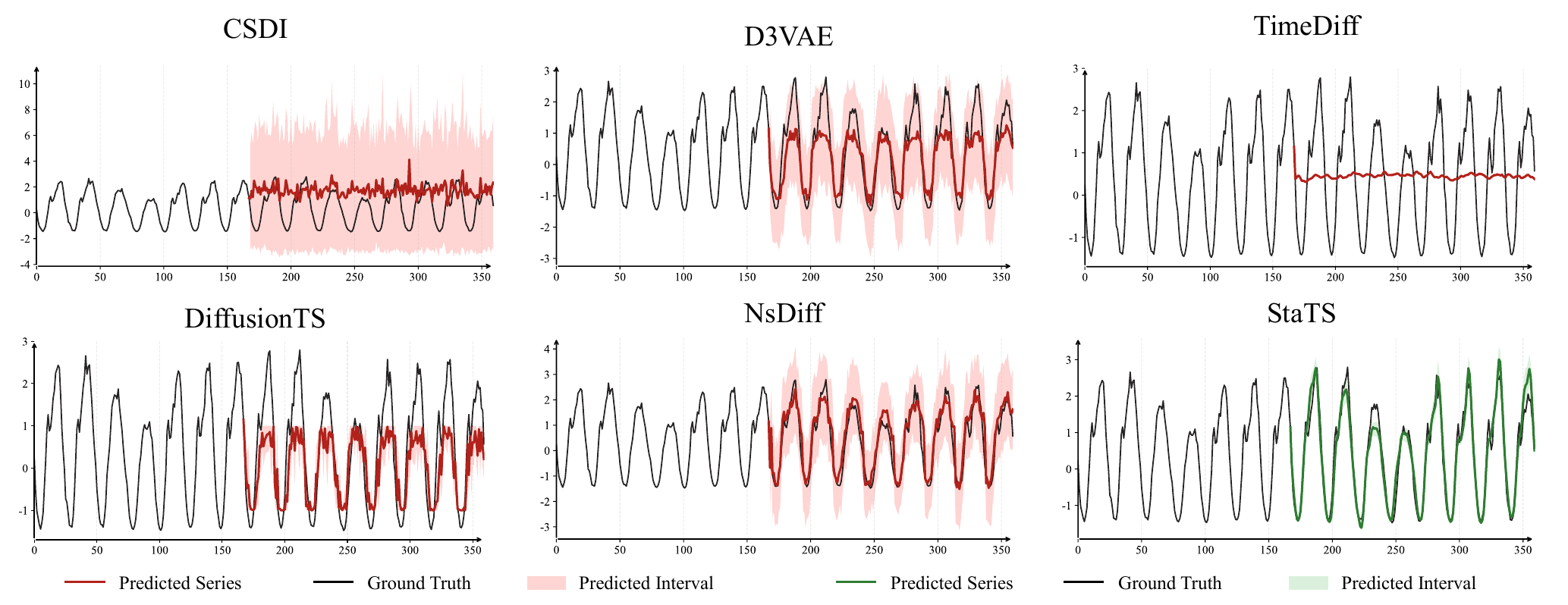}
  \caption{
  Visualization of min--max prediction intervals and point forecasts on the Traffic dataset.
  }
  \label{fig:visualization_traffic}
\end{figure}

\subsection{Reproducibility}
We release all code and notebooks at \url{https://anonymous.4open.science/r/StaTS-EC5F/}. The relevant data are provided in the supplementary materials.

\subsection{Key Hyperparameters and Default Settings}
\label{setting}
We report the key hyperparameters for the Spectral Trajectory Scheduler (STS) and the Spectral Conditioned Denoising Module (SCDM).
For STS, we use the following default weights:
$\lambda_{\mathrm{smooth}}=5$,
$\lambda_{\mathrm{init}}=0.5$,
$\lambda_{\mathrm{end}}=0.5$,
$\lambda_{\mathrm{bar}}=5\times 10^{-3}$,
$\lambda_{\mathrm{flatness}}=0.5$,
and we additionally set $\lambda_{\mathrm{obj}}=0.01$ for the forecasting-oriented objective that directly steers the learned trajectory toward improved prediction performance. These values are selected based on validation performance and are kept fixed across datasets unless otherwise stated. In SCDM, the spectral distortion estimation explicitly clips the normalized spectral distortion ratio to stabilize conditioning and avoid extremely large gating responses.
We set the clipping bounds to $r_{\min}=-10$ and $r_{\max}=10$ by default.
Moreover, the spectral conditioning branch adopts a multi-band design to capture coarse-grained frequency-dependent distortions; we use $B=2$ frequency bands as the default setting.

We tune the above hyperparameters on the validation split with a light grid centered around the default configuration:
$\lambda_{\mathrm{smooth}}\in\{1,2,5,10\}$,
$\lambda_{\mathrm{init}}\in\{0.1,0.5,1.0\}$,
$\lambda_{\mathrm{end}}\in\{0.1,0.5,1.0\}$,
$\lambda_{\mathrm{bar}}\in\{10^{-4},5\times 10^{-4},10^{-3},5\times 10^{-3},10^{-2}\}$,
$\lambda_{\mathrm{flatness}}\in\{0.1,0.5,1.0\}$,
and $\lambda_{\mathrm{obj}}\in\{10^{-3},5\times 10^{-3},10^{-2},5\times 10^{-2}\}$.
For the distortion clipping, we consider $(r_{\min},r_{\max})\in\{(-5,5),(-10,10),(-20,20)\}$, and for the number of frequency bands we search $B\in\{1,2,4\}$.

\section{Related Works}

\subsection{Diffusion models for Time Series Forecasting}

Diffusion models have emerged as a powerful paradigm owing to their strong capacity to capture temporal uncertainty. This field has attracted substantial attention in time series forecasting~\cite{fredformer, zhang2025towards, cheng2025comprehensive} in recent years. Early studies mainly explored how to adapt the diffusion paradigm to core time series forecasting architectures. TimeGrad~\cite{timegrad} was the first to integrate an autoregressive forecasting framework with a denoising diffusion process, enabling stepwise generation of future predictive distributions. In contrast, CSDI~\cite{csdi} adopted a non-autoregressive, score-based generative approach to more thoroughly capture spatiotemporal dependencies in multivariate sequences. Building on these foundations, later work gradually shifted toward improving predictive performance and practical applicability through stronger conditioning and tailored mechanisms. TimeDiff~\cite{timediff} improved pattern modeling via future mixup and autoregressive initialization. TSDiff~\cite{tsdiff} introduced a self-guiding strategy to enhance short-horizon accuracy. NsDiff~\cite{nsdiff} modeled non-stationarity with a location scale noise formulation, which supports more adaptive noise scheduling. CDPM~\cite{CDPM} adapts diffusion models to learn the full conditional distribution of future, enabling uncertainty-aware and flexible generation beyond point forecasts.

Meanwhile, diffusion modeling has been extended to continuous-time settings and more challenging sampling regimes. Score-based approaches such as ScoreGrad~\cite{scoregrad} use stochastic differential equations (SDEs) to support continuous-time forecasting, better handling irregular sampling and varying time intervals. In parallel, the high inference cost of diffusion models has motivated efficiency oriented designs. Latent diffusion~\cite{rombach2022high} reduces computation by running the diffusion process in a low-dimensional latent space, an idea validated in precipitation forecasting and time series prediction by LDCast~\cite{LDcast} and Latent Diffusion Transformers~\cite{latent1}, which achieve faster inference while maintaining generation quality.

Diffusion models elevate time series forecasting from point estimation to conditional distribution modeling. They enable the generation of diverse plausible futures and provide explicit uncertainty quantification, thereby enhancing forecast credibility and increasing their utility for risk-aware decision-making.

\subsection{Noise Schedules and Efficient Sampling}

\begin{figure}[htbp]
  \centering
  \includegraphics[width=0.9\linewidth]{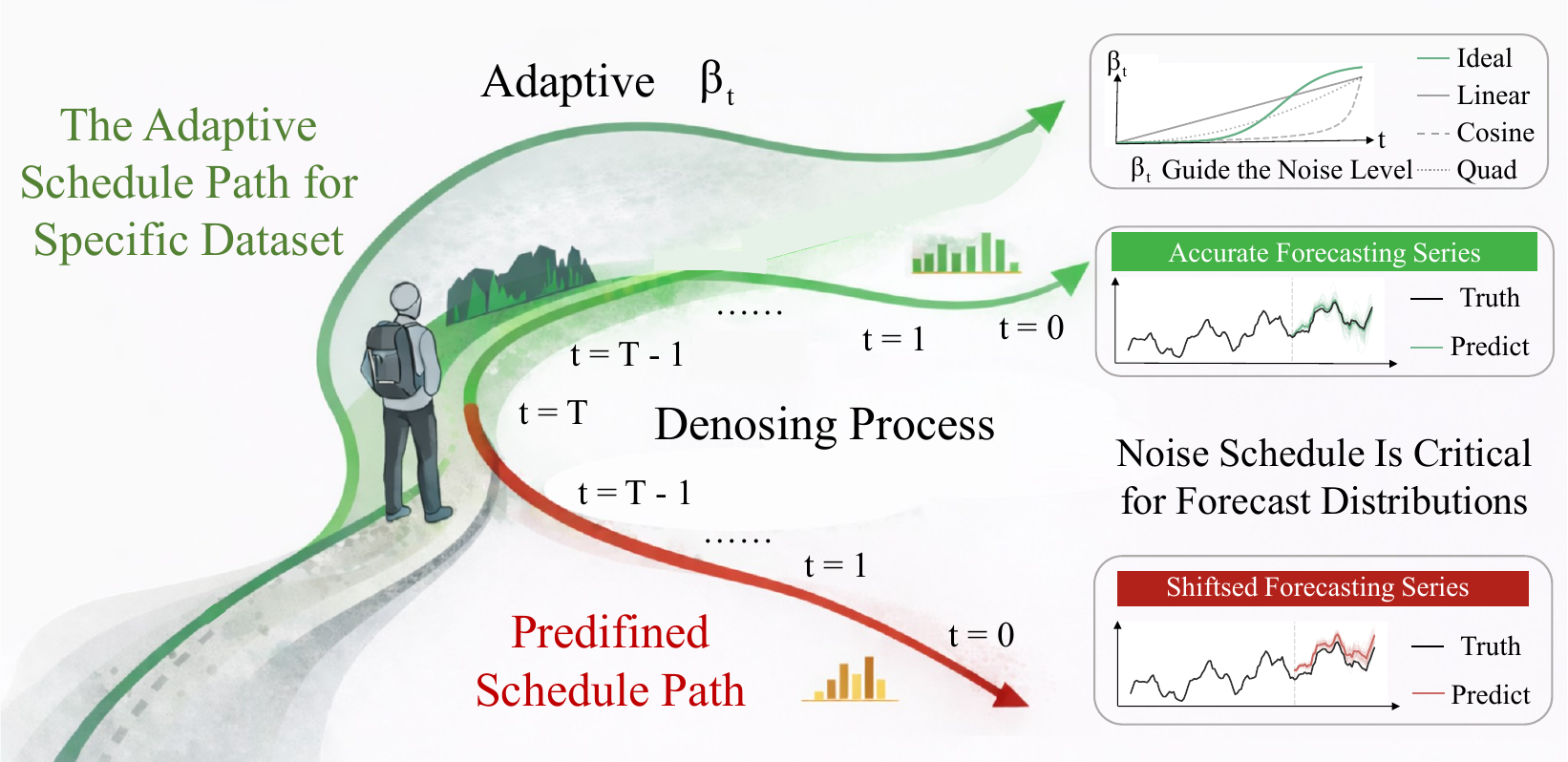}
  \caption{
  Improved scheduling yields more faithful predictive distributions and consistently better forecasting performance.
  }
  \label{fig:flsm}
\end{figure}

Noise scheduling determines the SNR evolution over diffusion steps, which in turn governs the extent of information preservation in the forward process and the difficulty of the reverse denoising. It therefore has a direct and often decisive impact on generation quality and sampling efficiency. Early diffusion models predominantly adopted hand-crafted schedules, such as linear or cosine schemes~\cite{R21}, as well as rescaled variants that enforce a zero terminal SNR~\cite{R22}. However, these choices are largely empirical and often transfer poorly across datasets and domains. Recent evidence further suggests that commonly used schedules and step-selection heuristics can be systematically suboptimal, owing to suboptimal SNR allocation and discretization error, which ultimately limits performance and numerical stability~\cite{R23}.

To overcome the limitations of hand-crafted schedules, a growing body of work learns or optimizes noise scheduling directly from data or task objectives. For example, MuLAN learns an adaptive noising process for multivariate diffusion from a variational modeling perspective~\cite{R24}. Learning to Schedule formulates schedule selection as an explicit optimization problem and uses dynamic programming or reinforcement learning to search for high-performing discretization schemes~\cite{R25}. To improve sampling efficiency, solver-aware timestep optimization chooses non-uniform timesteps by accounting for the numerical error of diffusion ODE solvers, which helps retain high generation quality with only a few sampling steps~\cite{R26}. Complementarily, Score-Optimal Diffusion Schedules models the trade-off between score estimation error and discretization cost, and derives objectives that better align with diffusion dynamics, enabling a more efficient allocation of computational resources~\cite{R27}.

In the field of time series, recent studies have started to incorporate structural priors into schedule design more explicitly. For instance, ANT adaptively selects noise schedules based on the statistical characteristics of the input series~\cite{ant}. Anisotropic diffusion models for time series further alter the diffusion transition operator, for example by introducing moving-average transitions, to better respect temporal dependency structures~\cite{R29}. Consequently, noise scheduling and timestep selection are no longer mere hyperparameter choices; they have become core algorithmic components that shape forecasting accuracy, uncertainty quantification, and compliance with practical inference budgets.

\section{Concurrent Submission Statement}
We also have a concurrent ICML submission that introduces PAFM, a perturbation-aware flow-based generator that injects localized perturbations into trajectories and leverages dual-path velocity modeling with flow-routing MoE decoding and velocity correction.
This framework StaTS instead focuses on diffusion-based probabilistic forecasting with learnable noise scheduling and a frequency-guided denoiser. The two submissions address different problem formulations and propose distinct methods; the note is included for transparency.

\end{document}